\documentclass[10pt,journal,cspaper,compsoc]{IEEEtran}
\ifCLASSOPTIONcompsoc
\else
  \usepackage{cite}
\fi
\ifCLASSINFOpdf
  \usepackage[pdftex]{graphicx}
\else
  \usepackage[dvips]{graphicx}
\fi
\usepackage[cmex10]{amsmath}
\usepackage{url}

\usepackage{amsthm}
\usepackage{multirow}
\usepackage{color}
\usepackage[table]{xcolor}

\def\bigO2{\mbox{${\cal O}$}}
\def\bigO{O}

\def\1n{\mathbf{1}_n}
\def\0{\mathbf{0}}
\def\1{\mathbf{1}}
\def\etal{{\em et al.}}

\def\bphi{\mbox{\boldmath $\phi$}}

\def\bSigma{{\bm \Sigma} }
\def\btheta{\mbox{\boldmath $\theta$}}

\def\bPhi{\mbox{\boldmath{$\Phi$}}}

\def\bSigma{\mbox{\boldmath{$\Sigma$}}}

\def\btheta{\mbox{\boldmath{$\theta$}}}

\def\A{{\bf A}}

\def\E{{\bf E}}

\def\H{{\bf H}}
\def\I{{\bf I}}
\def\J{{\bf J}}
\def\K{{\bf K}}

\def\M{{\bf M}}

\def\Q{{\bf Q}}
\def\R{{\bf R}}

\def\U{{\bf U}}

\def\X{{\bf X}}

\def\b{{\bf b}}
\def\c{{\bf c}}
\def\d{{\bf d}}

\def\f{{\bf f}}
\def\g{{\bf g}}
\def\h{{\bf h}}

\def\p{{\bf p}}

\def\t{{\bf t}}

\def\x{{\bf x}}
\def\y{{\bf y}}

\def\a1{\mbox{\bf a}_1}
\def\a2{\mbox{\bf a}_2}
\def\a3{\mbox{\bf a}_3}
\def\a4{\mbox{\bf a}_4}

\def\btheta{\mbox{\boldmath{$\theta$}}}

\def\figdir{./figures}
\def\ie{\textit{i.e.}}
\def\eg{\textit{e.g.}}

\def\tra{\text{trace}}
\def\notation{
Bold capital letters denote a matrix $\X$; bold lower-case letters denote a column vector $\x$.
All non-bold letters represent scalars.
$\x_i$ represents the $i$th column of the matrix $\X$.
$x_{ij}$ denotes the scalar in the $i^{th}$ row and $j^{th}$ column of the matrix $\X$.
$x_{j}$ denotes the scalar in the $j^{th}$ element of $\x$.
$\I_n \in \Re^{n \times n}$ is an identity matrix.
$\| \x \|_2=\sqrt{\x^T\x}$ denotes the Euclidean distance.
$\| \X \|_F = \tra(\X^T \X) = \tra(\X \X^T)$ designates the Frobenious norm of a matrix.
}
\theoremstyle{definition}
\newtheorem{mydef}{Definition}[section]
\newtheorem{theorem}{Theorem}[section]

\begin{document}
%
\title{Supervised Descent Method \\ for Solving  Nonlinear Least Squares \\ Problems in Computer Vision}
%
%
%
%

\author{Xuehan~Xiong, and~Fernando~De~la~Torre}

%
%

\markboth{Journal of \LaTeX\ Class Files,~Vol.~6, No.~1, January~2007}%
{Shell \MakeLowercase{\textit{et al.}}: Bare Demo of IEEEtran.cls for Computer Society Journals}
%


\IEEEcompsoctitleabstractindextext{%
\begin{abstract}
   Many computer vision problems (e.g., camera calibration, image alignment, structure from
    motion) are solved with nonlinear optimization methods. It is generally accepted that second order descent methods are the most robust, fast, and reliable approaches for nonlinear optimization of a general smooth function. However, in the context of computer vision, second order descent methods have two main drawbacks:   (1) the function might not be analytically differentiable and numerical approximations are impractical, and (2) the Hessian may be large and not positive definite. To address these issues, this paper proposes {\em generic descent maps}, which are average ``descent directions'' and rescaling factors learned in a supervised fashion.
    Using {\em generic descent maps}, we derive a practical algorithm - Supervised Descent Method (SDM) - for minimizing Nonlinear Least Squares (NLS) problems. During training, SDM learns a sequence of decent maps that minimize the NLS. In testing, SDM minimizes the NLS objective using the learned descent maps without computing the Jacobian or the Hessian. We prove the conditions under which the SDM is guaranteed to converge. We illustrate the  effectiveness and accuracy of SDM in three computer vision problems: rigid image alignment, non-rigid image alignment, and 3D pose estimation. In particular, we show how SDM achieves state-of-the-art performance in the problem of facial feature detection. The code has been made available at \url{www.humansensing.cs.cmu.edu/intraface}.
\end{abstract}

}

\maketitle

\IEEEdisplaynotcompsoctitleabstractindextext

%
\IEEEpeerreviewmaketitle

\section{Introduction}
%
%

%
%
%
%
\IEEEPARstart{M}{athematical} optimization plays a fundamental role in solving many problems in computer vision. This is evidenced by the significant number of papers using optimization techniques in any major computer vision conferences. 
Many important problems in computer vision, such as structure from motion, image alignment,
optical flow, or camera calibration can be posed as nonlinear optimization problems.  There are a large number of different approaches to solve these continuous nonlinear optimization problems based on first and second order methods, such as gradient descent~\cite{AbouDF10} for dimensionality reduction, Gauss-Newton for image alignment~\cite{Lucas-Kanade-IUW81,black96,DelaTorre-Nguyen-CVPR08} or Levenberg-Marquardt (LM)~\cite{More1977TLM} for structure from motion~\cite{buchanan}.

Despite many centuries of history, Newton's method and its variants (\eg, Quasi-Newton methods~\cite{berndt74,Broyden65-MC,LBFGS}) are regarded as powerfull optimization tools for finding local minima/maxima of smooth functions when second derivatives are available.  Newton's method makes the assumption that a smooth function $f(\x)$ can be well approximated by a quadratic function in a neighborhood of the minimum. If the Hessian is positive definite, the minimum can be found by solving a system of linear equations.
Given an initial estimate $\x_0 \in \Re^{p \times 1}$, Newton's method creates a sequence of updates as
\vspace{-0.05in}
\begin{equation}
\label{eq::newton-step}
\x_{k+1}=\x_k-\H^{-1}(\x_k)\J(\x_k),
\vspace{-0.015in}
\end{equation}
where $\H(\x_k) \in \Re^{p \times p}$ and $\J(\x_k) \in \Re^{p \times 1}$ are the Hessian matrix and Jacobian matrix evaluated at $\x_k$ (see notation \footnote{\notation}). In the case of Quasi-Newton methods, $\H^{-1}$ can be approximated by analyzing successive gradient vectors.
Newton-type methods have two main advantages over competitors.
First, they are guaranteed to converge provided that, in the neighborhood of the minimum, the Hessian is invertible and the minimizing function is Lipschitz continuous.
Second, the convergence rate is quadratic.

\begin{figure}[t]
\centering
\includegraphics[width=0.90\linewidth]{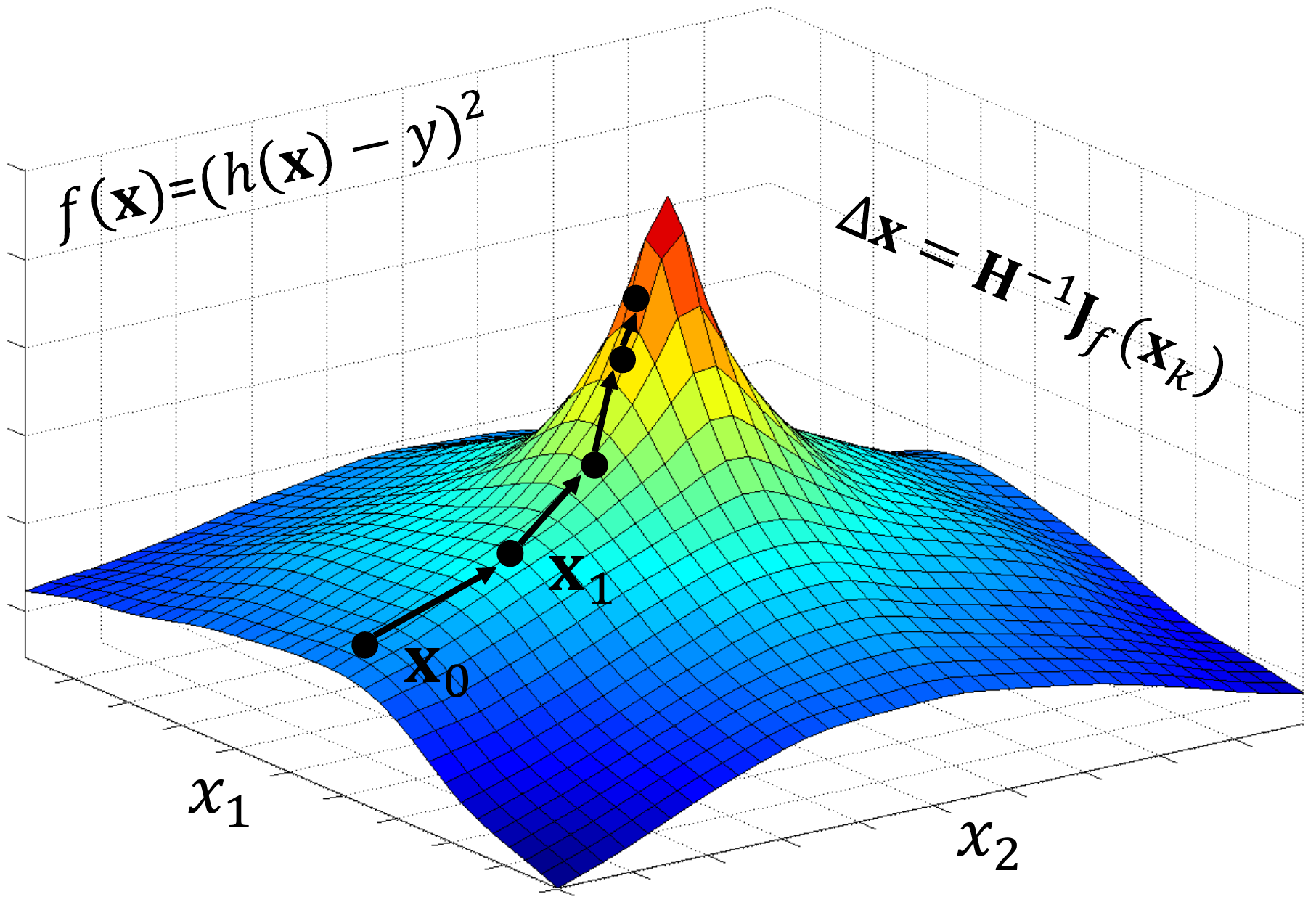}
\\[-2.5pt]
\includegraphics[width=0.90\linewidth]{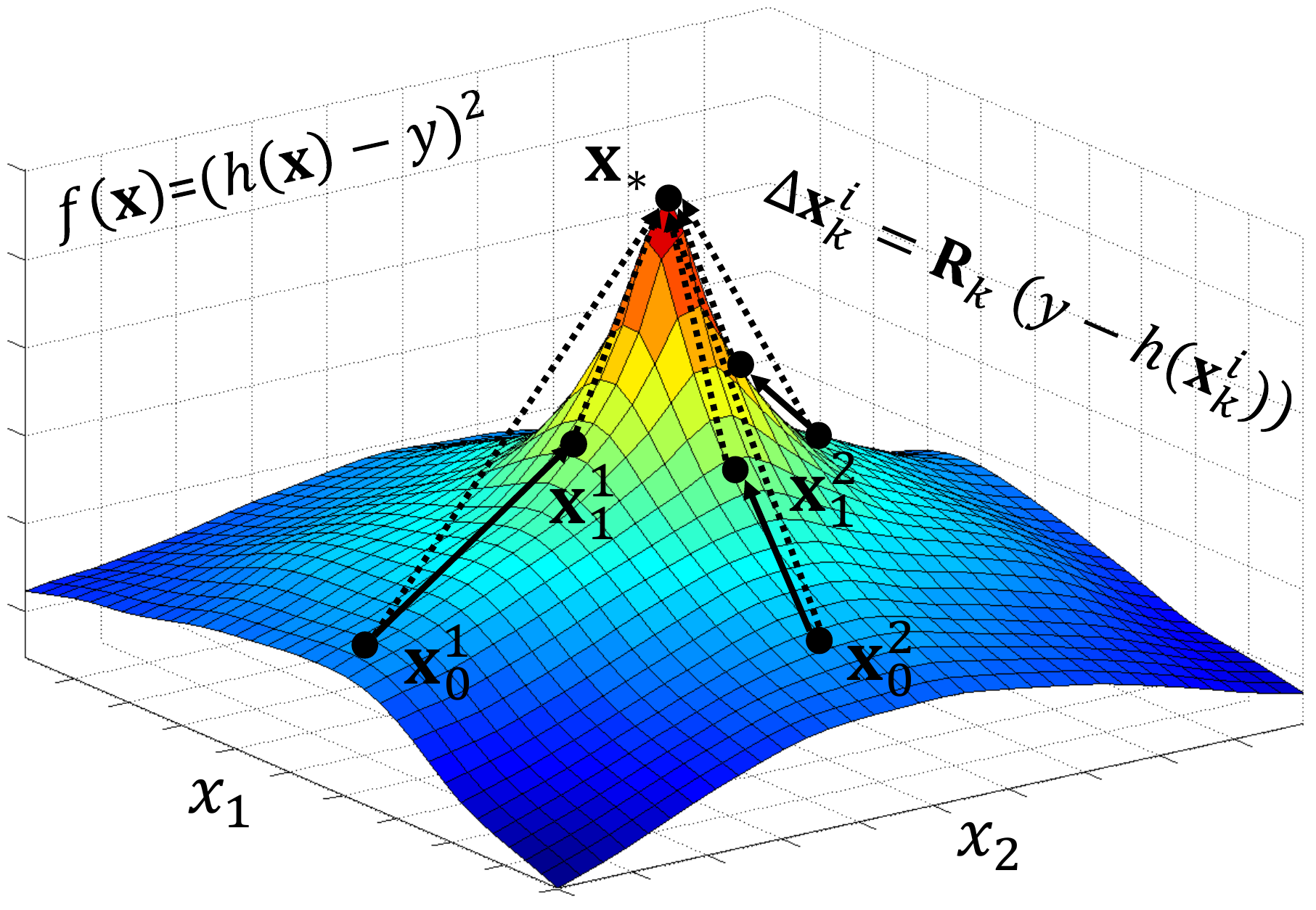}
\caption{a) Newton's method to minimize $f(\x)$. The z-axis is reversed for visualization purposes. b) SDM learns a sequence of generic descent maps $\{\R_k\}$ from the optimal optimization trajectories (indicated by the dotted lines). Each parameter update $\Delta \x^i$ is the product of $\R_k$ and a sample-specific component ($y-h(\x^i_k)$). }
\label{fig::main-idea}
\end{figure}
However, when applying  Newton's method to computer vision problems, three main problems arise:
(1) The Hessian is positive definite at the local minimum, but it may not be positive definite elsewhere; therefore, the Newton steps may not be in the descent direction. LM addressed this issue by adding a damping factor to the Hessian matrix. This increases the robustness of the algorithm but reduces the convergence rate.  (2) Newton's method requires the function to be twice differentiable. This is a strong requirement in many computer vision applications. For instance, the popular SIFT~\cite{Lowe-IJCV04} or HoG~\cite{Dalal-Triggs-CVPR05} features are non-differentiable image operators. In these cases, we can estimate the gradient or the Hessian numerically, but this is typically computationally expensive.
(3) The dimension of the Hessian matrix can be large; inverting the Hessian requires $O(p^3)$ operations and $O(p^2)$ in space, where $p$ is the dimension of the parameter to estimate.  Although explicit inversion of the Hessian is not needed using Quasi-Newton methods, it can still be computationally expensive to use these methods in computer vision problems.
In order to address previous limitations, this paper proposes the idea of learning descent directions (and rescaling factors) in a supervised manner with
a new method called Supervised Descent Method (SDM).

Consider  Fig.~\ref{fig::main-idea} where the goal is to minimize a Nonlinear Least Squares (NLS) function, $f(\x)=(h(\x)-y)^2$,
where $h(\x)$ is a nonlinear function , $\x$ is the vector of parameters to optimize, and $y$ is a known scalar. The $z$-axis has been reversed for visualization purposes. The top image shows the optimization trajectory with Newton's method. The traditional Newton update has to compute the Hessian and the Jacobian at each step. The bottom image illustrates the main idea behind SDM. In training, SDM starts by sampling a set of initial parameters $\{\x^i_0\}$ around the known minimum $\x_*$. For each sample, the optimal parameter update is also known in training (plotted in dotted lines from the point to $\x_*$). SDM will learn a sequence of updates to approximate
 the ideal trajectory. The SDM updates can be decomposed into two parts: a sample specific component (\eg, $h(\x^i_k)-y$)  and a generic Descent Map (DM) $\R_k$.
During training, SDM finds a series of $\{\R_k\}$ such that they gradually move all $\x_k^i$ towards $\x_*$. The updates in the first iteration using learned DM are plotted in black lines. Then, the optimal parameter updates are re-computed (shown in dotted lines) and used for learning the next DM. In testing, $f(\x)$ is optimized using the learned DMs without computing the Jacobian or the Hessian. SDM is particularly useful to minimize
a NLS problem where the template $y$ is the same in testing (\eg, tracking). However, this is not a limitation, and we will develop extensions of SDM for an unseen $y$.

We illustrate the benefits of SDM over traditional optimization methods in terms of speed and accuracy in three NLS problems in computer vision. First, we show how to extend the traditional Lucas-Kanade~\cite{lucas:an:81} 
tracker using HOG/SIFT features, resulting in a more robust tracker. Second, we apply SDM to the problem of facial feature detection and show state-of-the-art results. The code is publicly available.  Third, we illustrate how SDM improves current methods for 3D pose estimation.

\section{Theory of SDM}

This section provides the theoretical basis behind SDM. Section \ref{sec::background} reviews
background mathematical definitions.
Section~\ref{sec::one-dim} illustrates the ideas behind SDM and DMs using one-dimensional functions, and
Section~\ref{sec::high-dim} extends it to high-dimensional functions. Section~\ref{sec::soft-version} and Section \ref{sec::reverse-sdm} derive practical algorithms for different situations.

\begin{figure*}[t]
\centering
\addtolength{\tabcolsep}{-3pt}

\begin{tabular}{cccc}
\raisebox{6.5pt}{\includegraphics[width=0.3\linewidth]{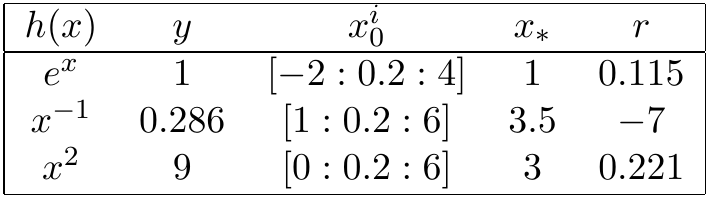}}&
\includegraphics[width=0.22\linewidth]{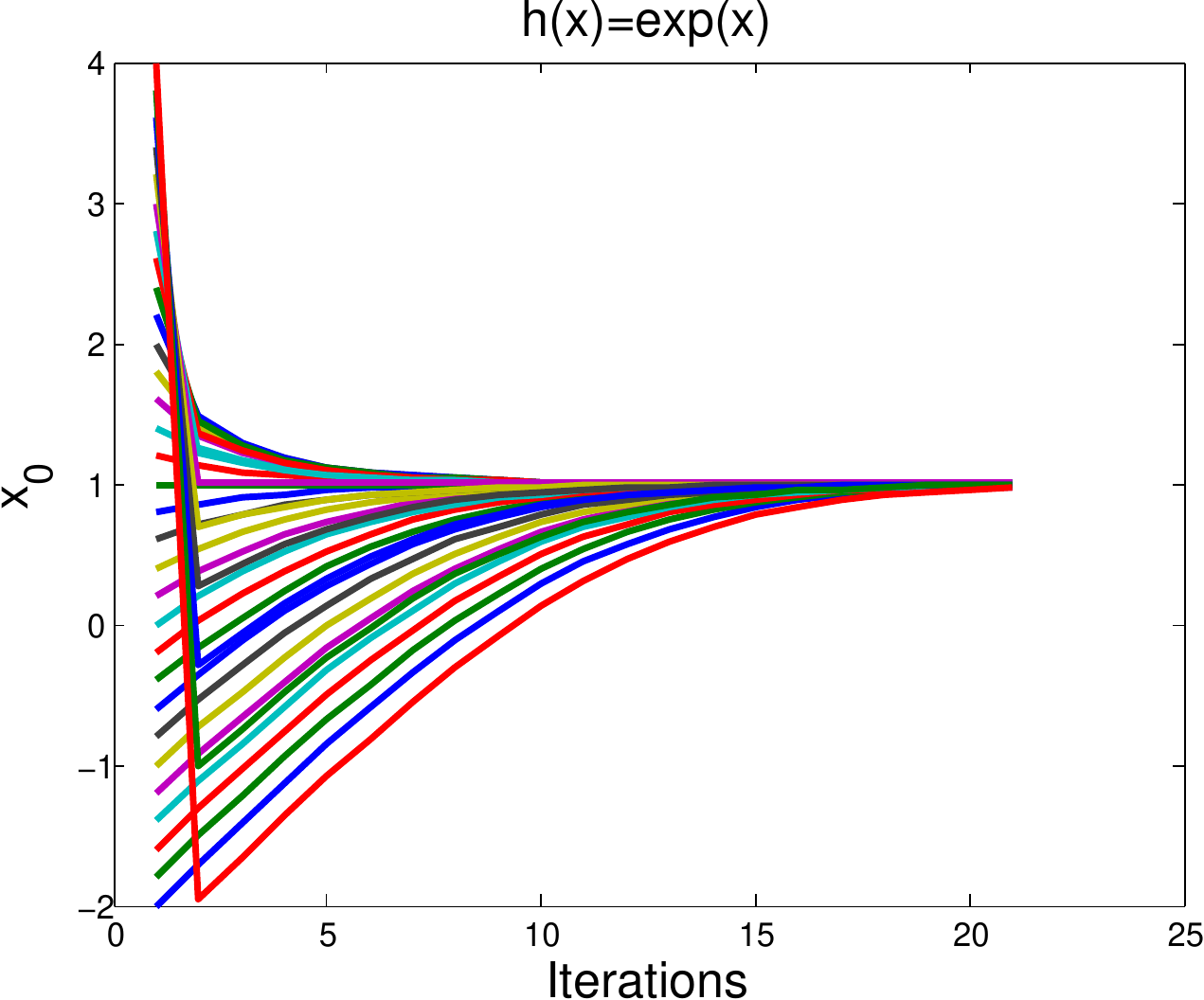}&
\includegraphics[width=0.22\linewidth]{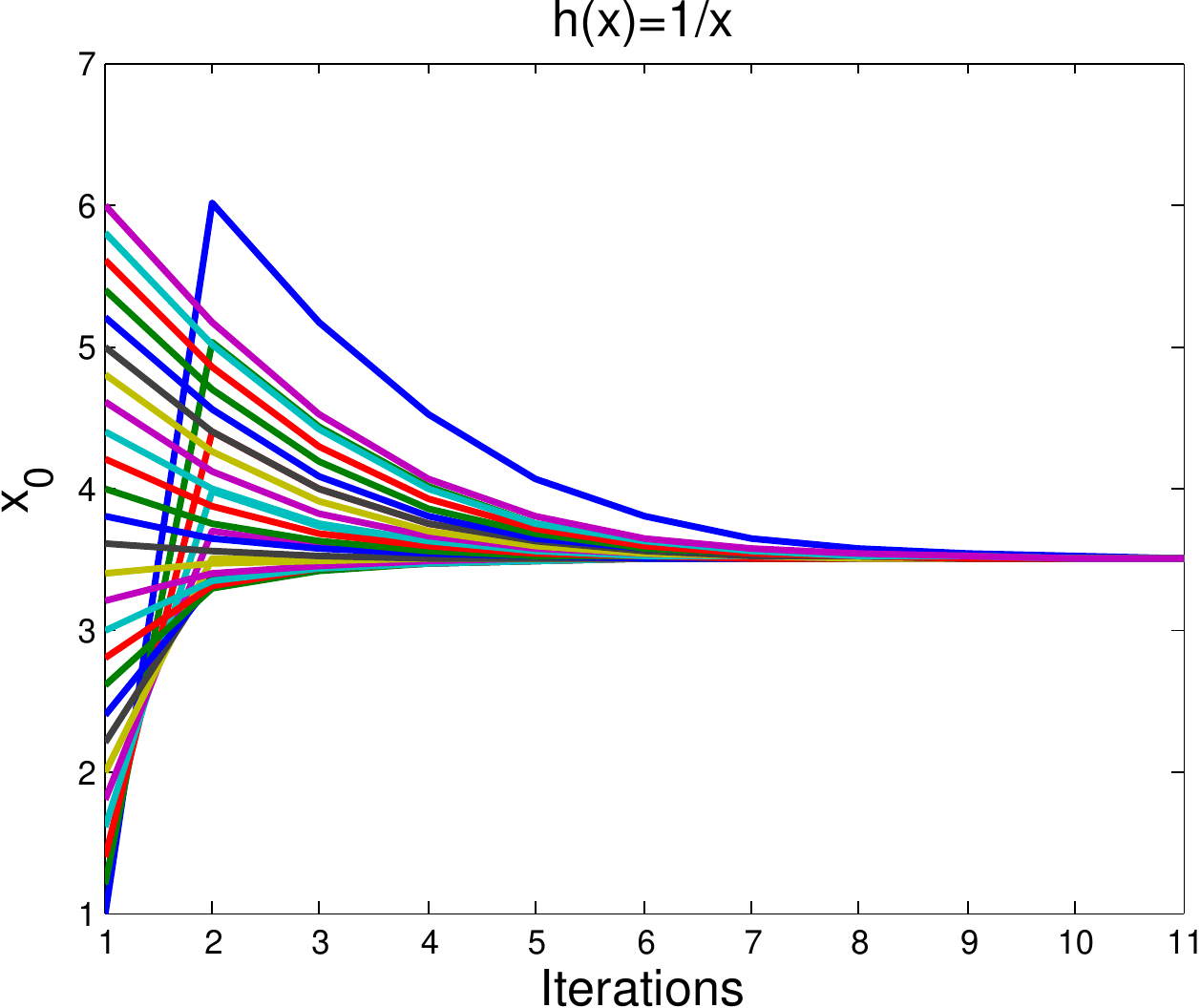}&
\includegraphics[width=0.22\linewidth]{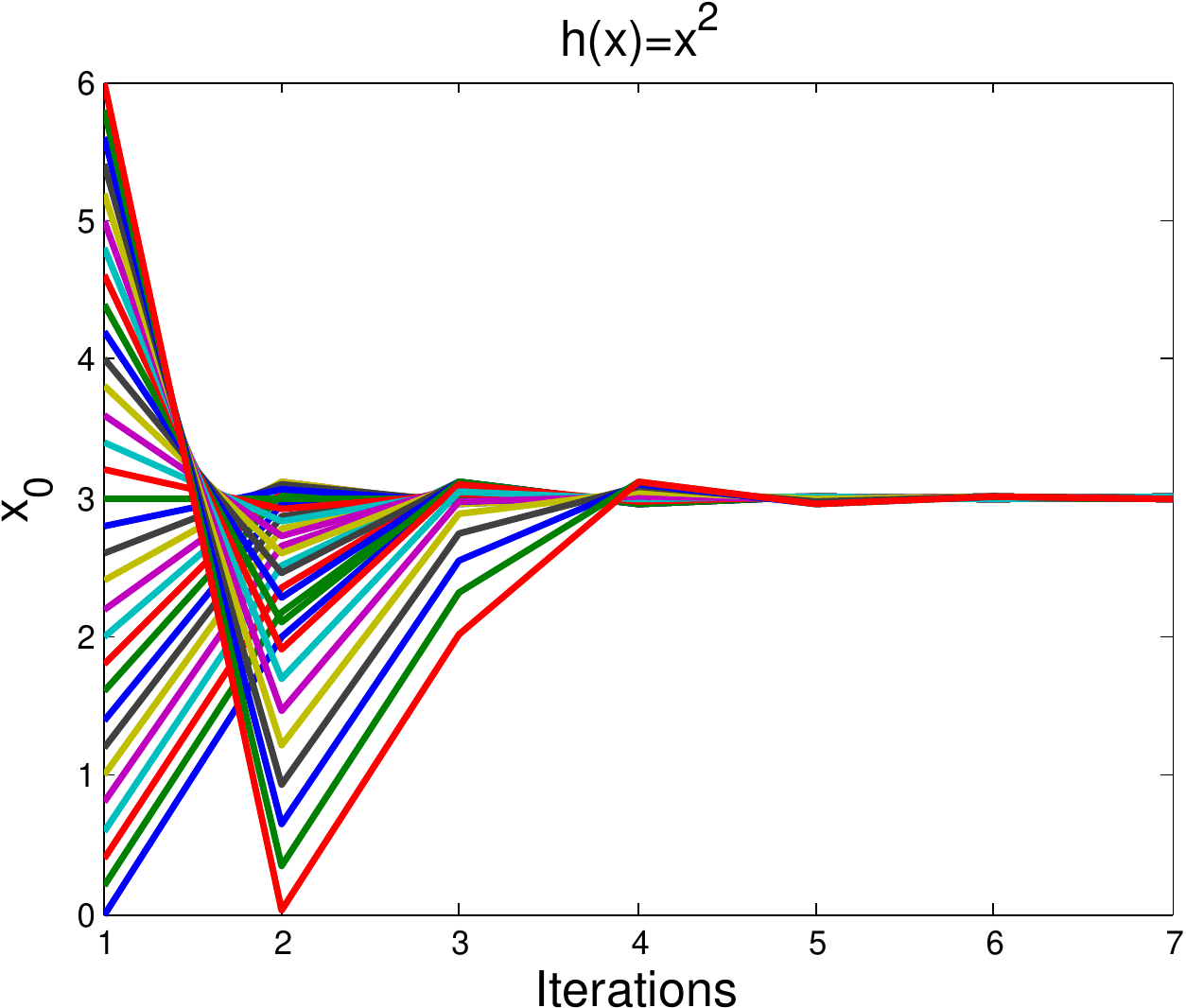}\\[-2.5pt]
(a)&(b)&(c)&(d)
\end{tabular}
\label{fig::analytic}
\caption{a) Experimental setup. The colon usage follows Matlab conventions. b,c,d) Convergence curves of $x_k$ using a generic DM on three functions. Different colors represent different initializations $x_0^i$.}
\end{figure*}
\subsection{Background (Definitions)}
\label{sec::background}
Before deriving SDM, we review two concepts, {\em Lipschitz continuous}~\cite{armijo1966} and {\em monotone operator}~\cite{Rockafellar76-SIAM}.

\begin{mydef}
A function $\f: \Re^n \to \Re^m$ is called {\em locally Lipschitz continuous} anchored at $\x_*$ if there exists a real constant $K \geq 0$ such that
$$
\|\f(\x) - \f(\x_*)\|_2 \leq K\|\x-\x_*\|_2, \forall \x \in U(\x_*),
$$
where $U(\x_*)$ represents a neighborhood of $\x_*$ and $K$ is referred as the Lipschitz constant.
\end{mydef}
 Note that this is different from the traditional Lipschitz continuous definition, which is defined for any pair of $\x$'s.
The advantage of our definition is that it is applicable to other distance metrics besides Euclidean.

\begin{mydef}
A function $\f: \Re^n \to \Re^n$ is called a {\em locally monotone operator} anchored at $\x_*$ if
$$
\langle \x - \x_*, \f(\x) - \f(\x_*)\rangle \geq 0, \forall \x \in U(\x_*),
$$
where $\langle \cdot \rangle$ represents the inner product operator. $\f$ is said to be a {\em strictly locally monotone operator} if
$$
\langle \x - \x_*, \f(\x) - \f(\x_*)\rangle > 0, \forall \x \in U(\x_*),
$$
\end{mydef}

\subsection{One-dimensional Case}
\label{sec::one-dim}
This section derives the theory for SDM in 1D functions. Given a 1D NLS problem,
\begin{equation}
\label{eq::nls}
\min_{x} f(x) = \min_x (h(x)-y)^2,
\end{equation}
where $h(x)$ is a nonlinear function and $y$ is a known scalar. Applying the chain rule to Eq.~\ref{eq::nls},
the gradient descent update rule yields
\begin{equation}
\label{eq::gd2}
x_k = x_{k-1} - \alpha h^{\prime}(x_{k-1}) (h(x_{k-1})-y).
\end{equation}
Gradient-based methods to optimize Eq.~\ref{eq::nls} follow Eq.~\ref{eq::gd2}, but have different ways to compute the step size $\alpha$.
 For example, Newton's method sets $\alpha = \frac{1}{f^{\prime\prime}(x_{k-1})}$. Computing the step size and gradient direction in high-dimensional spaces is computationally expensive
 and can be numerically unstable, especially in the case of non-differentiable functions, where finite differences are required to compute estimates of the Hessian and Jacobian.
 The main idea behind SDM is to avoid explicit computation of the Hessian and Jacobian and learn the ``descent directions'' ($\alpha h^\prime(x_{k-1})$) from training data.
 During training,  SDM samples many different initial configurations in the parameter space $\{x_0^i\}_i$ and
  learns a constant $r \sim \alpha h^\prime(x_{k-1})$, which drives all samples towards the optimal solution $x_*$. We define $r$ more formally below.
\begin{mydef}
A scalar $r$ is called a {\em generic DM} if there exists a scalar $0 < c < 1$ such that $\forall x \in U(x_*)$, $|x_* - x_{k}| \leq c|x_* - x_{k-1}|$. $x_{k}$ is updated using the following equation:
\begin{equation}
\label{eq::update-1d}
x_k = x_{k-1} - r(h(x_{k-1})-h(x_*)).
\end{equation}
\end{mydef}
The existence of a generic DM is guaranteed when both of the following conditions hold:
\begin{enumerate}
\item[1.] $h(x)$ is strictly monotonic around $x_*$.
\item[2.] $h(x)$ is locally Lipschitz continuous anchored at $x_*$ with $K$ as the Lipschitz constant.
\end{enumerate}
A detailed proof is presented in Section~\ref{sec::app} (Theorem~\ref{th::one-dim}). Interestingly, the above two conditions are closely related to the essence of a generic DM. The update rule (Eq.~\ref{eq::update-1d}) can be split into two terms: (1) $r \sim \alpha h^\prime(x_{k-1})$ (generic DM) that is
sample independent, and (2) $(h(x_{k-1})-y)$ that is sample dependent.  $r$ contains only part of the descent direction and needs to be multiplied by $h(x_{k-1})-y$ to produce a descent direction.
Condition 1 ensures that $h^\prime(x)$ does not change signs around $x_*$, while
condition 2 constrains the smoothness of the function, putting an upper bound on step sizes.

In Fig.~\ref{fig::analytic}, we illustrate how to minimize three different functions using a generic DM. The table in Fig.~\ref{fig::analytic}a describes our experimental setup for three different functions $h(x)$ with template $y$, different initial values $x_0^i$, the optimal value $x_*$ and the generic DM $r$. According to Theorem~\ref{th::one-dim}, the DM $r$ is set to be $\text{sign}(h^\prime)(\frac{2}{K}-\epsilon$), where $\epsilon$ is a small positive number, and the Lipschitz constant $K$ is computed numerically in a neighborhood of $x_*$. Figs.~\ref{fig::analytic}bcd plot the convergence curves for each function where $x_k$ is updated using Eq.~\ref{eq::update-1d}.
Note that SDM always converges to the optimal value $x_*$, regardless of the initial value $x_0^i$.

\subsection{Multi-dimensional Case}
\label{sec::high-dim}
This section extends the concept of generic DM to multi-dimensional functions, where $\h: \Re^n \to \Re^m$. For multi-dimensional NLS functions,
the gradient descent update in Eq. \ref{eq::gd2} becomes
\begin{equation}
\x_{k}  =  \x_{k-1} - \alpha\A\J_\h^\top(\x_{k-1})(\h(\x_{k-1})-\y)
\end{equation}
where $\J_\h(\x) \in \Re^{m\times n}$
is the Jacobian matrix, $\A_{n\times n}$ is the identity ($\I_n$) in first order gradient methods, and the inverse Hessian (or an approximation) for second-order methods, $\alpha$ is the step size.
A generic DM $\R$ exists if there exists a scalar $0<c<1$ such that
$$\|\x_* - \x_k\|_2 \leq c\|\x_*-\x_{k-1}\|_2. \forall \x \in U(\x_*)$$
The update rule of Eq.~\ref{eq::update-1d} becomes
\begin{equation}
\label{eq::update-hd}
\x_k =  \x_{k-1} - \R (\h(\x_{k-1})-\h(\x_*)).
\end{equation}
In section~\ref{sec::app}, we prove that the existence of a generic DM if both of the following conditions hold:
\begin{enumerate}
\item[1.] $\g(\x)=\R\h(\x)$ is a strictly locally monotone operator anchored at the optimal solution $\x_*$.
\item[2.] $\h(\x)$ is locally Lipschitz continuous anchored at $\x_*$ with $K$ as the Lipschitz constant.
\end{enumerate}

\subsection{Relaxed SDM}
\label{sec::soft-version}
In the previous section, we have stated the conditions that ensure the existence of a generic DM. However, in many computer vision applications, such conditions may not hold. In this section, we introduce a relaxed version of the generic DM and  we derive a practical algorithm.  Previously, a single matrix $\R$ was used for {\em all} samples. In this section, we extend SDM to learn a
sequence of $\{\R_k\}$ that moves the {\em expected value} of $\x_k$ towards the optimal solution $\x_*$. The relaxed SDM is a sequential algorithm that learns such DMs from training data.

Let us denote $X$ to be a random variable representing the state of $\x$. In the first iteration ($k=0$), we assume that $X_0$ is coming from a known distribution $X_0 \sim P_0$ and use lower case $p_0$ to represent its probability density function. The first DM $\R_0$ is computed by minimizing the expected loss between the true state and the predicted states, given by
\begin{align}
 & \E\|\x_* - X_{1}\|^2_2 = \E\|\x_* - X_{0} + \R_{0}(\h(X_{0})-\h(\x_*))\|^2_2 \nonumber\\
=& \int_{\x_0}\|\x_* - \x_{0} + \R_{0}(\h(\x_{0}^i)-\h(\x_*))\|^2_2 p_0(\x_0) d\x_0\nonumber\\
\label{eq::monte-carlo}
\approx & \sum_{i}\|\x_* - \x^i_{0} + \R_{0}(\h(\x_{0}^i)-\h(\x_*))\|^2_2.
\end{align}
Here we have used Monte Carlo sampling to approximate the integral, and $\x_0^i$ is drawn from the distribution $P_0$. $\x_*,\h(\x_*)$ are known in training and minimizing Eq. \ref{eq::monte-carlo} is simply a linear least squares problem,  which can be solved in closed-form.

It is unlikely that the first generic DM can achieve the desired minimum in one step for all initial configurations.  As in Newton's
method, after an update, we recompute a new generic map (\ie, a new Jacobian and Hessian in Newton's method). In iteration $k$, with $\R_{k-1}$ estimated, each sample $\x_{k-1}^i$ is updated to its new location $\x^i_k$ as follows:
\begin{equation}
\label{eq::hd-update}
\x_k^i =  \x_{k-1}^i - \R_{k-1} (\h(\x_{k-1}^i)-\h(\x_*)).
\end{equation}
This is equivalent to drawing samples from the conditional distribution $P(X_k|X_{k-1}=\x_{k-1}^i)$. We can use the samples to approximate the expected loss as follows:
\begin{align}
& \E\|\x_* - X_{k+1}\|_2^2 = \E\|\x_* - X_{k} + \R_{k}(\h(X_{k})-\h(\x_*))\|^2_2 \nonumber\\
= & \int_{\x_k}\|\x_* - \x_{k} + \R_{k}(\h(\x_{k})-\h(\x_*))\|^2_2 p(\x_{k}|\x_{k-1}) d\x_k \nonumber\\
\label{eq::monte-carlo-t}
\approx & \sum_{i}\|\x_* - \x^i_{k} + \R_{k}(\h(\x_{k}^i)-\h(\x_*))\|^2_2.
\end{align}
Minimizing Eq.~\ref{eq::monte-carlo-t} yields the $k^{th}$ DM. During learning, SDM alternates between minimizing Eq. \ref{eq::monte-carlo-t} and updating Eq. \ref{eq::hd-update} until convergence. In testing, $\x_k$ is updated recursively using Eq.~\ref{eq::hd-update} with learned DMs.

\subsection{Generalized SDM}
\label{sec::reverse-sdm}

Thus far, we have assumed that $\y=\h(\x_*)$ is the same for training as for testing (e.g., template tracking).
This section presents generalized SDM, which addresses the case where $\y$ differs between training and testing.  There are two possible
scenarios: $\y$ is known and $\y$ is unknown. In this section, for simplicity, we will only discuss the case when $\y$ is known and its
applications to  pose estimation. Section~\ref{sec::sdm-face-alignment} discusses the case when $\y$ is unknown and its application to face alignment.

 The goal of 3D pose estimation is to estimate the 3D rigid transformation ($\x$) given a 3D model of the object and its projection ($\y$), and $\h$ is the projection function. See section~\ref{sec::pose}.
In training, the pairs of 3D transformation and image projection are known. In testing, the projection (2D landmarks) is known, but different from those used in training. To address this problem,
we reverse the order of learning. Unlike SDM, the reversed SDM during training always starts at the same initial point $\x_0$ and samples different $\{\x_*^i\}$ around it.
At the same time, for each sample, we compute $\y^i = \h(\x_*^i)$.  The training procedure remains the same as stated in the previous section, except we replace $\h(\x_*)$ in Eq.~\ref{eq::hd-update} with $\y^i$,
\begin{equation}
\label{eq::rsdm-update}
\x_k^i =  \x_{k-1}^i - \R_{k-1} (\h(\x_{k-1}^i)-\y^i)),
\end{equation}
and we replace $\x_*, \h(\x_*)$ in Eq. \ref{eq::monte-carlo-t} with $\x_*^i, \y^i$,
\begin{equation}
\label{eq::rsdm-learning}
\sum_{i}\|\x_*^i - \x^i_{k} + \R_{k}(\h(\x_{k}^i)-\y^i))\|^2_2.
\end{equation}
In testing, we start SDM with the same initial value $\x_0$ that we used in training.

\section{Rigid Tracking}

This section illustrates how to apply SDM to the problem of tracking rigid objects. In particular,
we show how we can extend the classical Lucas-Kanade (LK) method~\cite{lucas:an:81} to efficiently track in HoG~\cite{Dalal-Triggs-CVPR05} space. To the best of our knowledge, this is the first algorithm to perform alignment in HoG space.

\subsection{Previous Work (Lucas-Kanade)}
\label{sec::LK}
The Lucas-Kanade (LK) tracker is one of the early and most popular computer vision
trackers due to its efficiency and simplicity. It formulates image alignment as a NLS problem.
Alignment is achieved by finding the motion parameter $\p$ that minimizes
\begin{equation}
\min_{\p} \ ||\d(\f(\x,\p)) -\t(\x)||_2^2,
\label{eq::LK}
\end{equation}
where $\t(\x)$ is the template, $\x = [x_1, y_1, ... x_l, y_l]^\top$ is a vector containing the coordinates of the pixels to detect/track, and $\f(\x,\p)$ is a vector with entries $[u_1, v_1, ..., u_l, v_l]^\top$ representing a geometric transformation. In this section, we limit the transformation to be affine. That is,  $(u_i, v_i)$ relates to $(x_i, y_i)$ by
$$
\begin{bmatrix}
u_i \\
v_i
\end{bmatrix}
=
\begin{bmatrix}
p_1 & p_2\\
p_4 & p_5
\end{bmatrix}
\begin{bmatrix}
x_i \\
y_i \end{bmatrix}
+
\begin{bmatrix}
p_3 \\
p_6
\end{bmatrix}.
$$
The $i^{th}$ entry of $\d(\f(\x,\p))$ is the pixel intensity of the image $\d$ at $(u_i, v_i)$.
Minimizing Eq. \ref{eq::LK} is a NLS problem because the motion parameters are nonlinearly related to the pixel values (it is a composition of two functions).

Given a template (often the initial frame), the LK method uses Gauss-Newton to minimize Eq.~\ref{eq::LK} by
linearizing the motion parameters around an initial estimate $\p_0$. The $k^{th}$ step of the LK update is
$$
\Delta\p = \H_k^{-1} \left(\frac{\partial\d}{\partial\f}\frac{\partial \f}{\partial \p_k}\right)^\top \big(\t(\x)-\d(\f(\x,\p_k)) \big),
$$
where $\H_k = (\frac{\partial\d}{\partial\f}\frac{\partial \f}{\partial \p_k})^\top(\frac{\partial\d}{\partial\f}\frac{\partial \f}{\partial \p_k})$ is the Gauss-Newton approximation of the Hessian. The motion parameter is then updated as $\p_{k+1} = \p_k + \Delta\p$.

\subsection{An SDM Solution}
\label{sec::sdm-rigid-tracking}
The LK method provides a mathematically sound solution for image alignment problems. However, it is not robust to illumination changes.  Robustness can be achieved
by aligning images w.r.t. some image descriptors instead of pixel intensities, \ie,
\begin{equation}
\label{eq::LK-SIFT}
\min_{\p} \ ||\h(\d(\f(\x,\p))) -\h(\t(\x))||_2^2,
\end{equation}
where $\h$ is some image descriptor function (HoG, in our case).
The LK (Gauss-Newton) update for minimizing \ref{eq::LK-SIFT} can be derived as follows:
$$
\Delta\p = \H_k^{-1} \left(\frac{\partial\h}{\partial\p_k}\right)^\top \big(\h(\t(\x))-\h(\d(\f(\x,\p_k))) \big).
$$
However, the update can no longer computed efficiently: HoG is a non-differentiable image operator, and thus the Jacobian ($\frac{\partial\h}{\partial\p_k}$) has to be estimated numerically at each iteration.

In contrast, SDM minimizes Eq.~\ref{eq::LK-SIFT} by replacing the computationally expensive term $\H_k^{-1} (\frac{\partial\h}{\partial\p_k})^\top$ with a pre-trained DM $\R$, and yields the following update:
\begin{equation}
\label{eq::sdm-template}
\Delta\p = \R_k (\h(\t(\x))-\h(\d(\f(\x,\p_k)))).
\end{equation}
Each update step in SDM is very efficient, mainly consisting of one affine image warping and one HoG descriptor computation of the warped image.
One may notice the inconsistency between Eq. \ref{eq::sdm-template} and the SDM update we previously introduced in Eq. \ref{eq::hd-update}. In the following, we will show the equivalence of the two.

The template can be seen as the HoG descriptors extracted from the image under an identity transformation, $\t(\x) = \t(\f(\x,\p_*))$, where
$$
\p_* =
\begin{bmatrix}
1 & 0 & 0 & 0 & 1 & 0
\end{bmatrix}^\top.
$$
Under the assumption that only affine deformation is involved,
the image $\d$ on which we perform tracking can be interpreted as the template image under an unknown affine transformation $\tilde{\p}$:
$$
\d(\x) = \t(\f(\x,\tilde{\p})).
$$
In Eq. \ref{eq::sdm-template}, image $\d$ warped under the current parameter $\p_k$ can be re-written as
\begin{equation}
\label{eq::affine-transform}
\d(\f(\x,\p_k)) = \t(\f(\f(\x,\tilde{\p}),\p_k)).
\end{equation}
The composition of two affine transformations remains affine, so Eq. \ref{eq::affine-transform} becomes
$$
\d(\f(\x,\p_k)) = \t(\f(\x,\widehat{\p})),
$$
where $\widehat{\p}$ an unknown affine parameter that differs from $\tilde{\p}$.
Therefore, we can re-write Eq. \ref{eq::sdm-template} as
\begin{equation}
\label{eq::sdm-template2}
\Delta\p = \R (\h(\t(\f(\x,\p_*)))-\h(\t(\f(\x,\widehat{\p})))).
\end{equation}
Eq. \ref{eq::sdm-template2} can be further simplified to follow the same form of Eq. \ref{eq::hd-update}:
$$
\Delta\p = \R (\g(\p_*)-\g(\widehat{\p})),
$$
where $\g = \h \circ \t \circ \f$.

In our implementation, we use Eq. \ref{eq::sdm-template} as the update rule instead of Eq. \ref{eq::sdm-template2} because $\p_k$ is a known parameter w.r.t image $\d$ and $\widehat{\p}$ is unknown w.r.t the template image $\t$. The descent maps are learned in the neighborhood of $\p_*$, but as tracking continues, the motion parameter may deviate greatly from $\p_*$. When tracking a new frame, before applying SDM the image is first warped back using the motion parameter estimated in the previous frame so that the optimization happens within a neighborhood of $\p_*$. Therefore, after SDM finishes, we warp back the estimated $\Delta\p$ using the same parameter.

The training of SDM involves sampling initial motion parameters and solving a sequence of linear systems (detailed in section \ref{sec::soft-version}).
We sample $\{\p_0^i\}_i$ around $\p_*$ and those samples are used for approximating the expectation expressed in Eq. \ref{eq::monte-carlo-t}. The details of how we generate initial samples are described in section \ref{sec::exp::rigid-tracking}.

\section{Nonrigid Detection and Tracking}
\label{sec::face-alignment}
In the previous section, we showed how SDM can be used for aligning regions of images that undergo an affine motion. This section extends SDM to detect and track nonrigid objects. In particular, we will show how SDM achieves state-of-the-art performance in facial feature detection and tracking.

\subsection{Previous Work}
\label{sec::fa::prev-work}
This section reviews existing work on face alignment.

{\bf Parameterized Appearance Models (PAMs)}, such as
Active Appearance Models~\cite{Cootes-et-al-PAMI01,DelaTorre-Nguyen-CVPR08,baker-ijcv-04}, Morphable Models~\cite{blanz,poggio2}, Eigentracking~\cite{black96}, and template tracking~\cite{Lucas-Kanade-IUW81,tzimiroREPFAICCV2011}
build an object appearance and shape representation by performing Principal Component Analysis (PCA) on a set of manually labeled data.
Fig.~\ref{fig::init}a illustrates an image labeled with $p$ landmarks ($p=66$ in this case).
After the images are aligned with
Procrustes~\cite{Dryden98}, a shape model is learned by performing PCA on the registered shapes. A linear combination of $k_s$ shape bases $\U^s \in \Re^{2p\times k_s}$ can
reconstruct (approximately) any aligned shape in the training set. Similarly, an appearance model $\U^a \in \Re^{m \times k_a}$ is built by performing PCA on the texture. Alignment is achieved by finding the motion parameter $\p$ and appearance coefficients $\c^a$ that best align the image w.r.t. the subspace $\U^a$, \ie,
\begin{equation}
\min_{\c^a,\p} \ ||\d(\f(\x,\p)) -\U^a\c^a||_2^2,
\label{eqn:PAM}
\end{equation}
In the case of the LK tracker, $\c^a$ is fixed to be $\mathbf{1}_{k_a}$ and $\U^a$ is a subspace that contains a single vector, the reference template.
The notation follows that of Section~\ref{sec::LK};
$\x = [x_1, y_1, \dots, x_l, y_l]^\top$ contains the coordinates of the pixels to track, and $\f(\x,\p)$ is a vector denoted by $[u_1, v_1, ..., u_l, v_l]^\top$ that now includes both affine and nonrigid transformation. That is, $(u_i, v_i)$ relates to $(x_i, y_i)$ by
$$
\begin{bmatrix}
u_i \\
v_i
\end{bmatrix}
=
\begin{bmatrix}
a_1 & a_2\\
a_4 & a_5
\end{bmatrix}
\begin{bmatrix}
x_i^s \\
y_i^s \end{bmatrix}
+
\begin{bmatrix}
a_3 \\
a_6
\end{bmatrix}.
$$
Here
$$
[x_1^s, y_1^s, ... x_l^s, y_l^s]^\top = \overline{\x} + \U^s\c^s,
$$
where $\overline{\x}$ is the mean shape face,
 $\mathbf{a}, \c^s$ are the affine and nonrigid motion parameters respectively, and $\p =[\mathbf{a}; \c^s]$. Similar to the LK method, PAMs algorithms~\cite{black96,baker-ijcv-04,Cootes-et-al-PAMI01,DelaTorre-Nguyen-CVPR08} optimize Eq.~\ref{eqn:PAM} using the Gauss-Newton method.
A more robust formulation of (\ref{eqn:PAM}) can be achieved by either replacing the L-2
norm with a robust error function \cite{black96,Baker_2003_4302} or by performing matching on robust features, such as gradient orientation \cite{tzimiro2012-PAMI}.

{\bf Discriminative approaches} learn a mapping from image features to
motion parameters or landmarks. Cootes \etal~\cite{Cootes-et-al-PAMI01} proposed to fit AAMs by learning a linear regression between the increment of motion parameters $\Delta\p$ and the
appearance differences $\Delta\d$. The linear regressor is a numerical approximation of the Jacobian~\cite{Cootes-et-al-PAMI01}. Following this idea, several discriminative methods that learn a mapping from $\d$ to $\Delta\p$ have been proposed.
Gradient Boosting, first introduced by Friedman~\cite{Friedman01}, has become one of the most popular regressors in face alignment because of its efficiency and ability to model nonlinearities.
Saragih and G\"ocke \cite{Saragih-Goecke-ICCV07} and Tresadern \etal~\cite{Tresadern10}
showed that using boosted regression for AAM discriminative fitting significantly improved over the original linear formulation.
Doll\'{a}r \etal~\cite{Dollar10} incorporated ``pose indexed features'' to
the boosting framework, where features are re-computed at the latest estimate of the landmark locations in addition to learning a new weak regressor at each iteration.
Beyond gradient boosting,  Rivera and Martinez~\cite{Rivera12} explored kernel regression to map from image features directly to landmark locations, achieving surprising results for low-resolution images.
Recently, Cootes \etal~\cite{Cootes-ECCV2012} investigated Random Forest regressors in the context of face alignment. At the same time,
S\'anchez \etal~\cite{Sanchez-ECCV2012} proposed to learn a regression model in the continuous domain to efficiently and uniformly sample the motion space. In the context of tracking, Zimmermann \etal~\cite{Zimmermann-pami09} learned a set of independent linear predictors for
different local motions and then chose a subset of them during tracking. Unlike PAMs, a major problem of discriminative approaches is that the cost function being minimizing is unclear, making these algorithms difficult to analyze theoretically. This paper is the first to formulate a precise cost function for discriminative approaches.

{\bf Part-based deformable models} perform alignment by maximizing the posterior likelihood of part locations given an image. The objective
function is composed of the local likelihood of each part times a global shape prior. Different methods typically vary the optimization methods or the shape prior.
Constrained Local Models (CLM)~\cite{Cristinacce08} model this prior similarly as AAMs, assuming all faces lie in a linear subspace spanned by PCA bases.
Saragih \etal~\cite{Saragih-ICCV2009} proposed a nonparametric representation to model the posterior likelihood and the resulting optimization method is reminiscent of mean-shift.
In~\cite{Belhumeur11}, the shape prior was modeled nonparametrically from training data. Recently, Saragih~\cite{Saragih-CVPR11}
derived a sample specific prior to constrain the output space providing significant improvements over the original PCA prior. Instead of using a global model, Huang et al.~\cite{Metaxas_11} proposed to build separate Gaussian models for each part (\eg, mouth, eyes) to preserve more detailed local shape deformations.
Zhu and Ramanan \cite{Zhu12} assumed that the face shape is a tree structure (for fast inference), and used a part-based model for face detection, pose estimation, and facial feature detection.

\subsection{An SDM Solution}
\label{sec::sdm-face-alignment}
\begin{figure}[t]
\centering
\addtolength{\tabcolsep}{2pt}
\begin{tabular}{cc}
\includegraphics[width=0.35\linewidth]{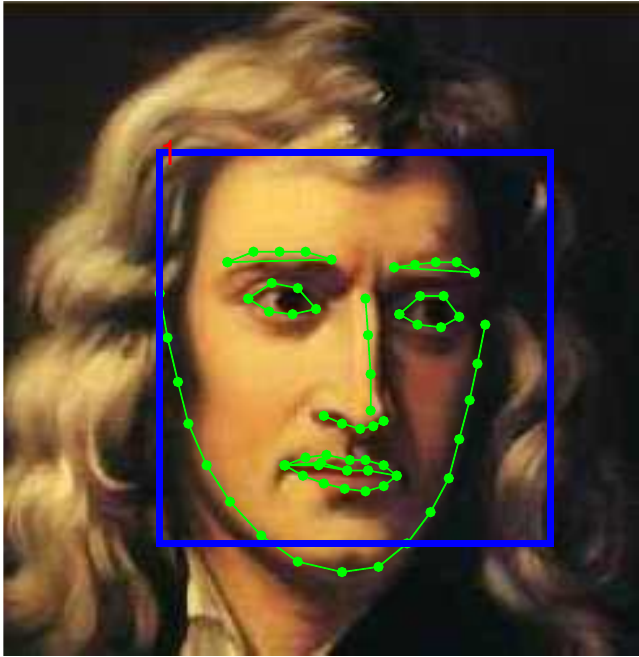}&
\includegraphics[width=0.35\linewidth]{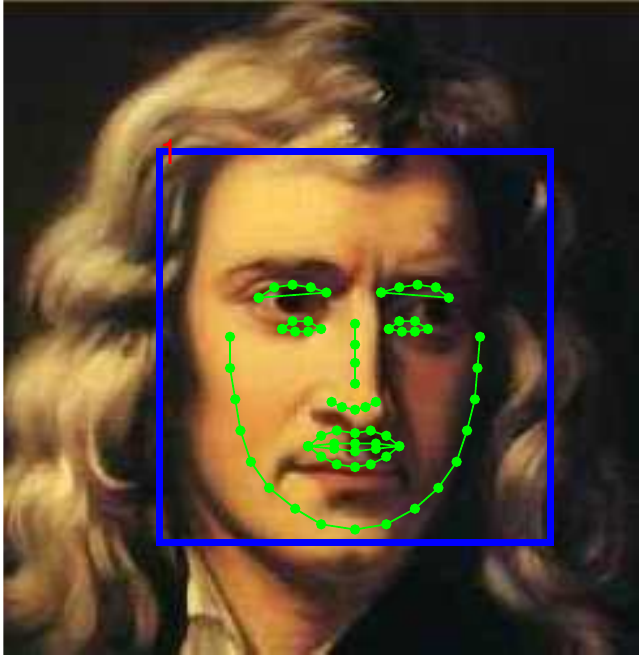}\\
(a) $\x_*$  & (b) $\x_0$
\end{tabular}
\caption{a) Manually labeled image with 66 landmarks. Blue outline indicates face detector.
b) Mean landmarks, $\x_0$, initialized using the face detector. }
\label{fig::init}
\end{figure}
Similar to rigid tracking in section~\ref{sec::sdm-rigid-tracking}, we perform face alignment in the HoG space. Given an image $\d \in \Re^{m \times 1}$ of $m$ pixels, $\d(\x) \in \Re^{p \times 1}$ indexes $p$ landmarks in the image. $\h$ is a nonlinear feature extraction function and $\h(\d(\x)) \in  \Re^{128p \times 1}$ in the case of extracting HoG features.
In this setting, face alignment can be framed as minimizing the following NLS function over landmark coordinates $\x$:
\begin{equation}
\label{eq::face-alignment}
f(\x) = \|\h(\d(\x))-\y_*\|_2^2,
\end{equation}
where $\y_*=\h(\d(\x_*))$ represents the HoG values computed on the local patches extracted from the manually labeled landmarks.
During  training, we assume that the correct $p$ landmarks (in our case $p=66$) are known,
and we will refer to them as $\x_*$ (see Fig.~\ref{fig::init}a).
Also, to simulate the testing scenario, we run the face detector on the training images to
provide an initial configuration of the landmarks ($\x_0$), which corresponds to an average shape (see Fig.~\ref{fig::init}b).

Eq.~\ref{eq::face-alignment} has several fundamental differences with previous work on PAMs (Eq.~\ref{eqn:PAM}). First, in Eq.~\ref{eq::face-alignment}, we do not learn any model of shape or appearance beforehand from training data. Instead, we align the image w.r.t. a template $\y_*$. For the shape, we optimize the landmark locations $\x \in \Re^{2p \times 1}$ directly.
Recall that in traditional PAMs, nonrigid motion is modeled as a linear combination
of shape bases learned by performing PCA on a training set. Our shape formulation is able to generalize better to untrained situations (\eg, asymmetric facial gestures). Second, we use HoG features extracted from patches around the landmarks to achieve a representation robust to illumination changes.

\begin{figure*}[t]
\centering
\setlength{\tabcolsep}{2pt}
\begin{tabular}{ccccc}
\includegraphics[width=0.19\linewidth]{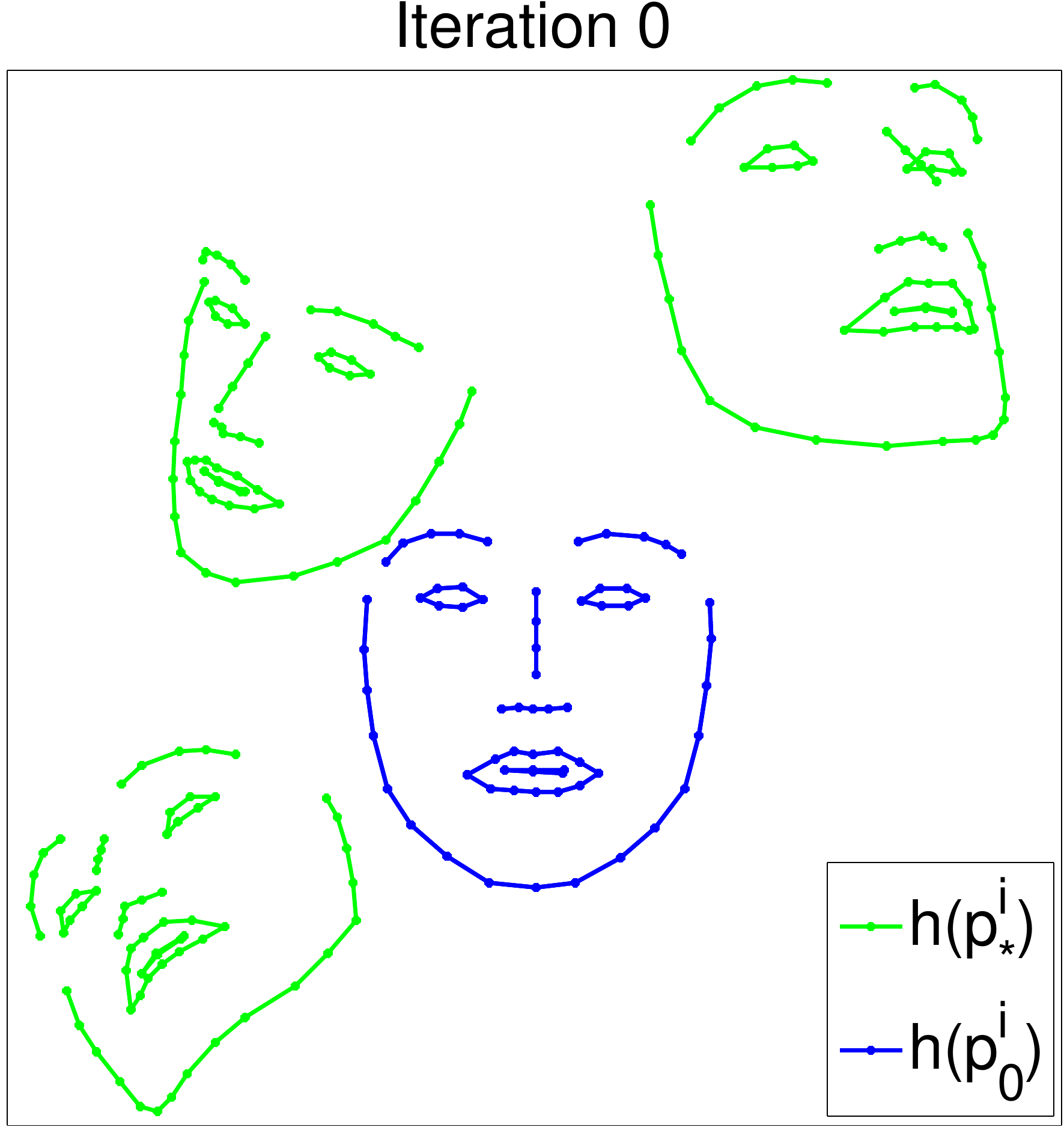}&
\includegraphics[width=0.19\linewidth]{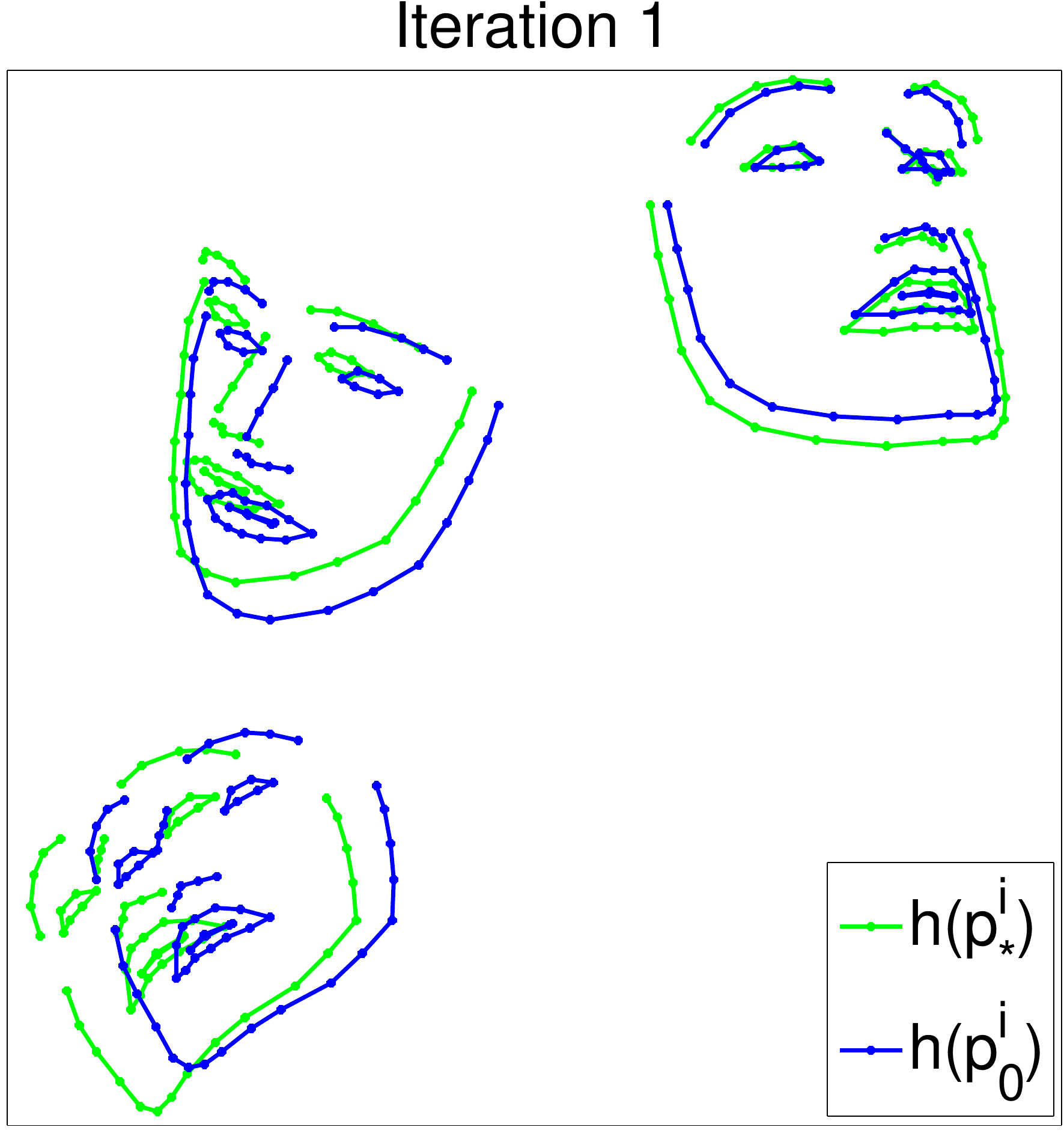}&
\includegraphics[width=0.19\linewidth]{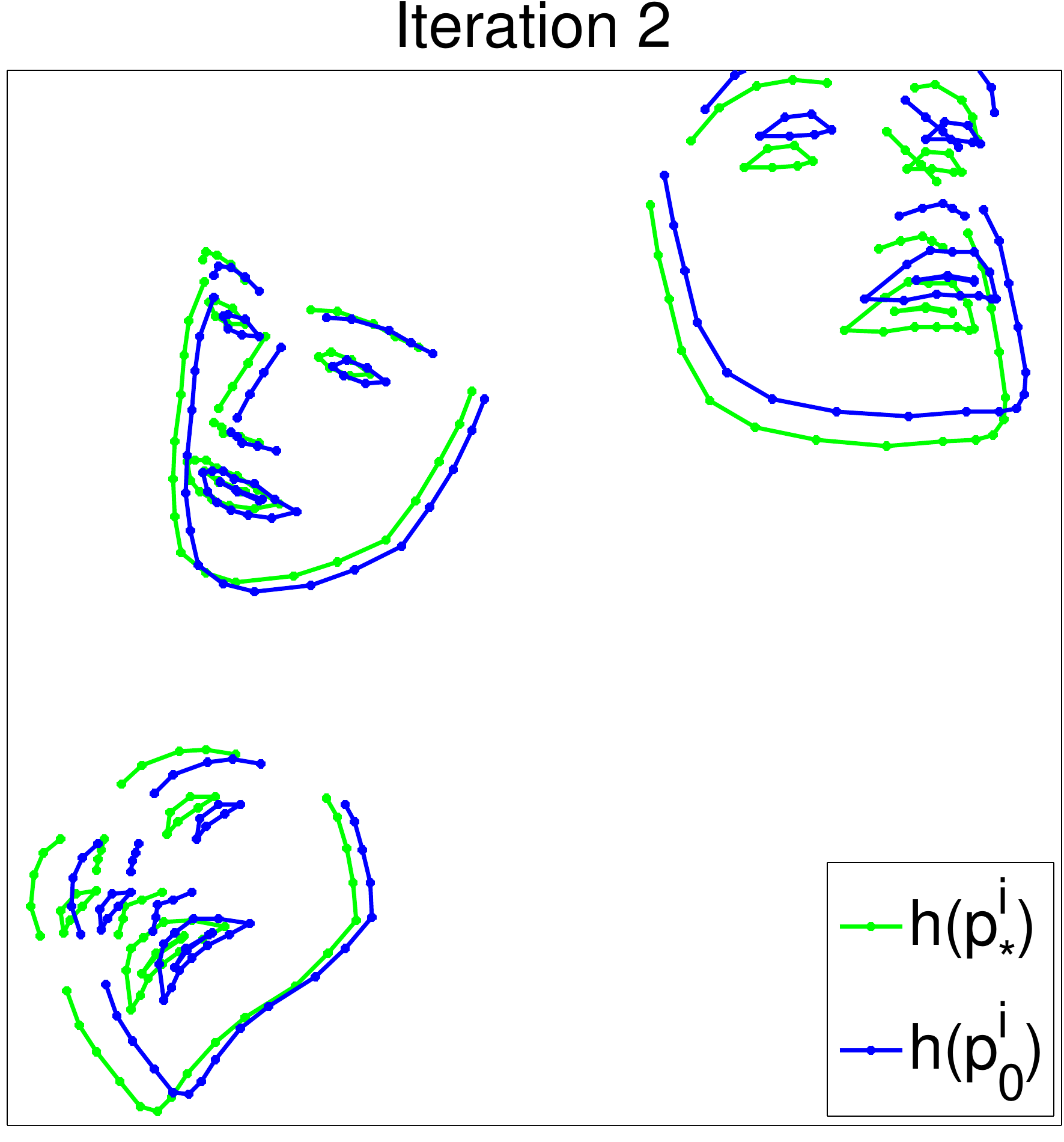}&
\includegraphics[width=0.19\linewidth]{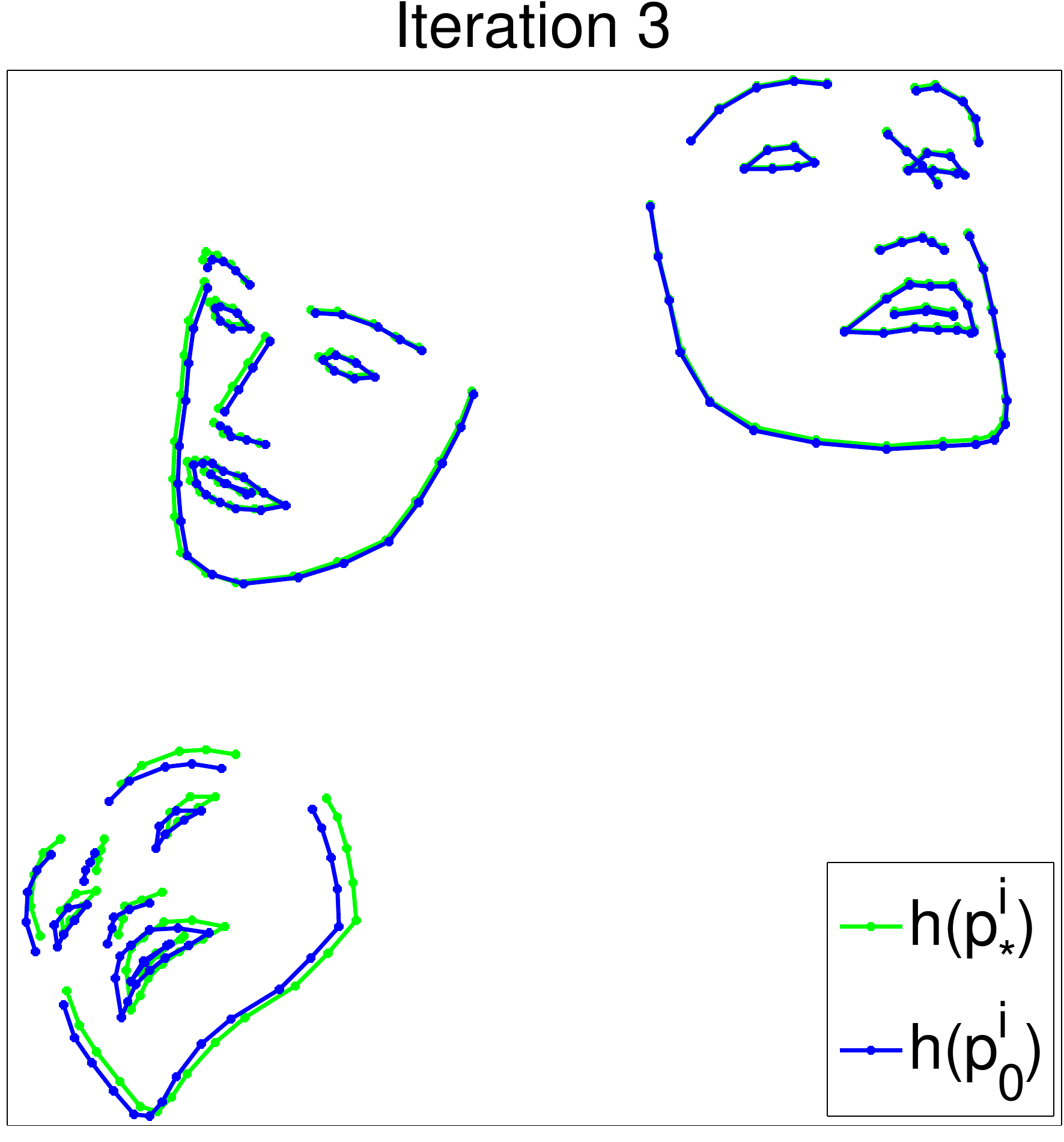}&
\includegraphics[width=0.19\linewidth]{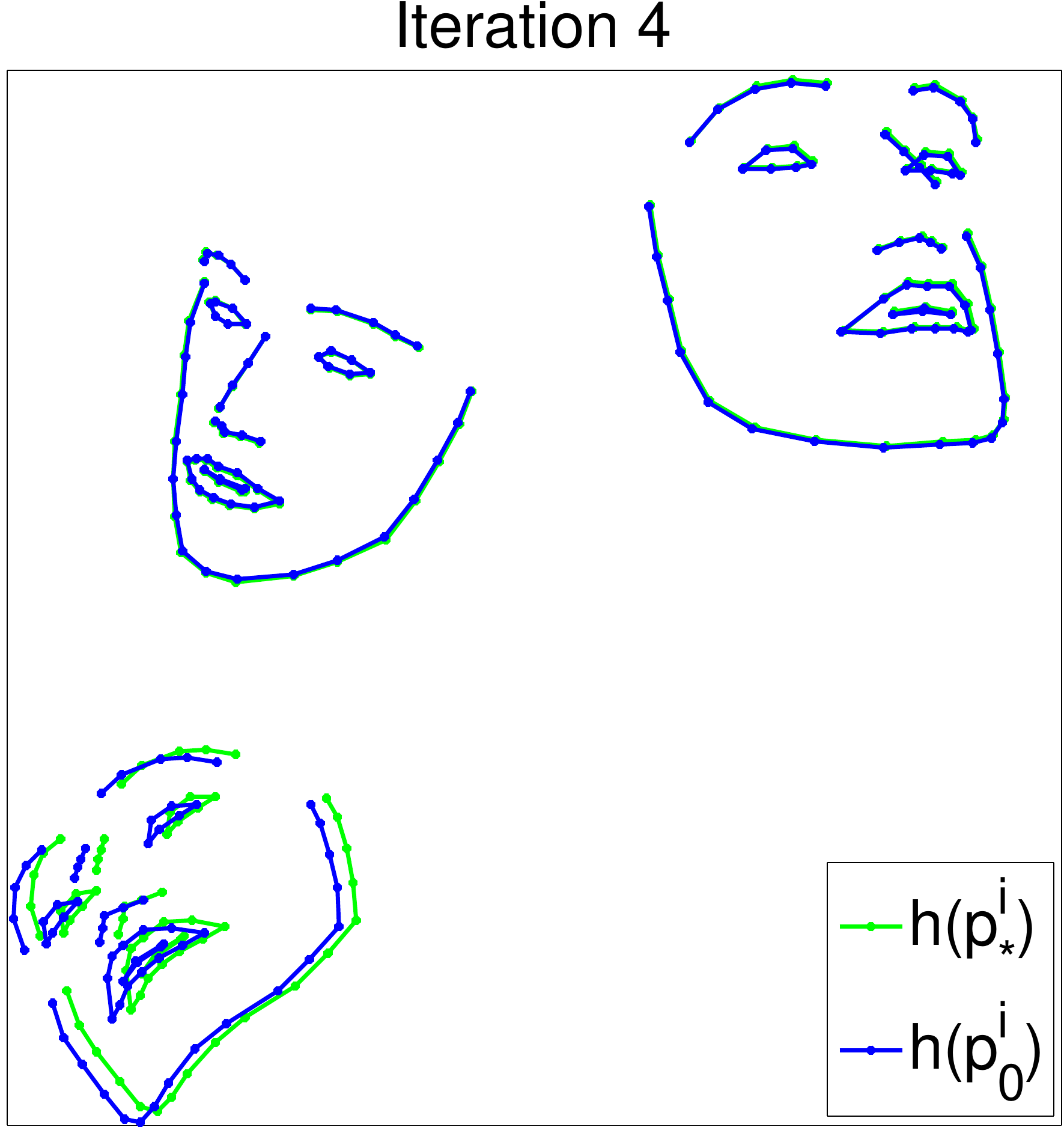}
\end{tabular}
\caption{Three examples of SDM minimizing the reprojection errors through each step. Blue outlines represent the image projections $\h(\p^i_k)$ under the current parameter estimates $\p^i_k$. Green outlines are the projections rendered under the ground truth parameters $\p_*^i$.}
\label{fig::pose}
\end{figure*}
{\bf Unknown $\y_*$}: In face alignment, the $\y_*$ used for testing is unknown and different from those used for training (\ie, the test subject is not one of the training subjects). Furthermore, the function $\h$ is parametrized not only by $\x$, but also by the images (\ie, different subjects or different conditions of subjects). Therefore, SDM learns an additional bias term $\b_k$ to represent the average of $\R_k\h^i(\x_*^i)$ during training. The training step modifies Eq.~\ref{eq::hd-update} to minimize the expected loss over all initializations and images, where the expected loss is given by
\begin{equation}
\label{eq::sdm-learning2}
\sum_{i,j}\|\x_* - \x^{i,j}_{k} + \R_{k}\h(\d^i(\x_{k}^{i,j}))-\b_k\|^2_2.
\end{equation}
We use $i$ to index images and $j$ to index initializations. The update of Eq.~\ref{eq::hd-update} is thus modified to be
\begin{equation}
\label{eq::hd-update2}
\x_k^{i,j} =  \x_{k-1}^{i,j} - \R_{k-1} \h(\d^i(\x_{k-1}^{i,j}))+\b_{k-1}.
\end{equation}
Despite the modification, minimizing Eq.~\ref{eq::hd-update2} is still a linear least squares problem. Note that we do not use $\y_*^i$ in training, although they are available. 
In testing, given an unseen image $\tilde{\d}$ and an initial guess of $\tilde{\x}_0$, $\tilde{\x}_k$ is updated recursively using Eq.~\ref{eq::hd-update2}. If $\tilde{\d}, \tilde{\x}_0$ are drawn from the same distribution that produces the training data, each iteration is guaranteed to decrease the expected loss between $\tilde{\x}_k$ and $\x_*$. 

\subsection{Online SDM}
\label{sec::online-sdm}
SDM may have poor performance on an unseen sample that is dramatically different from those in the training set. It would be desirable to incorporate this new sample into the existing model without re-training. This section describes such a procedure that updates an existing SDM model in an online fashion.

Assume that one is given a trained SDM model, represented by $\{\R_k,\b_k,\bSigma_k^{-1}\}$,
where $\bSigma_k = \bPhi_k \bPhi_k^\top$ and $\bPhi_k$ is the data matrix used
in training the $k^{th}$ descent map.
For a given new face image $\d$ and labeled landmarks $\x_*$, one can compute
the initial landmark perturbation $\Delta\x_0 = \x_*-\x_0$ and the feature
vector extracted at $\x_0$, $\bphi_0 = \h(\d(\x_0))$. Using the well known
recursive least squares algorithm~\cite{RLS}, SDM can be  re-trained by iterating the following three steps:
\begin{enumerate}
\item[1.] Update inverse covariance matrix $\bSigma_k^{-1}$
\begin{equation}
\label{eq::variance}
\bSigma_k^{-1} \leftarrow \bSigma_k^{-1}-\bSigma_k^{-1} \bphi_k (w^{-1}+\bphi_k^\top
\bSigma_k^{-1}\bphi_k)^{-1}\bphi_k^\top\bSigma_k^{-1}.
\end{equation}
\item[2.] Update the generic descent direction $\R_k$
\begin{equation*}
\R_k \leftarrow \R_k + (\Delta\x_k - \R_k\bphi_k) w \bphi_k^\top \bSigma_k^{-1}.
\end{equation*}
\item[3.] Generate a new sample pair $(\Delta\x_{k+1},\bphi_{k+1})$ for re-training in the next iteration
\begin{align*}
\Delta\x_{k+1} & \leftarrow  \Delta\x_k - \R_k \bphi_k\\
\bphi_{k+1} &  \leftarrow \h(\d(\x_*+\Delta\x_{k+1})).
\end{align*}
\end{enumerate}
Setting the weight to be $w=1$  treats every sample equally. For different applications,
one may want the model to emphasize the more recent samples. For example,
SDM with exponential forgetting can be implemented with a small modification of Eq. \ref{eq::variance}:
\begin{equation*}
\label{eq::variance-exp}
\bSigma_k^{-1} \leftarrow \lambda^{-1}[\bSigma_k^{-1}-\bSigma_k^{-1} \bphi_k (\lambda+\bphi_k^\top
\bSigma_k^{-1}\bphi_k)^{-1}\bphi_k^\top\bSigma_k^{-1}],
\end{equation*}
where $0<\lambda<1$ is a discount parameter. Assuming $n$ data points come in order,
the weight on the $i^{th}$ sample is $\lambda^{n-i}$.
Above, we do not explain the update formula for the bias term $\b_k$,
since it is often incorporated into $\R_k$ by augmenting the feature vector with $1$.
Note that in Eq. \ref{eq::variance}, the term in parentheses is a scalar.
Since no matrices need to be inverted, our re-training scheme is very efficient, consisting of only a few matrix multiplications and feature extractions.

\section{3D Pose Estimation}
\label{sec::pose}
In the two applications we have shown thus far, the optimization parameters lie in $\Re^n$ space. In this section, we will show how SDM can also be used to optimize parameters such as a rotation matrix, which belongs to the $SO(3)$ group. 
 
 The problem of 3D pose estimation can be described as follows. Given the 3D model of an object represented as $3D$ points $\M \in \Re^{3\times n}$, 
 its projection $\U \in \Re^{2\times n}$ on an image, and the  intrinsic camera parameters $\K \in \Re^{3\times 3}$, the goal is to estimate the object pose (3D rotation $\Q \in \Re^{3\times 3}$ and translation $\t\in \Re^{3\times 1}$).\footnote{$\Q$ is used for rotation matrix to avoid conflict with DM $\R_k$} This is also known as extrinsic camera parameter calibration.

\subsection{Previous Work}
Object pose estimation is a well-studied problem. If the model and image projection points are perfectly measured, this problem can be solved in closed-form by finding the perspective projection matrix~\cite{Roberts63,Sutherland74-IEEE}. The projection matrix maps the 3D model points to image points in homogeneous coordinates. Since it has 11 unknowns, at least six correspondences are required. However, these approaches are very fragile to noise. Lowe~\cite{Lowe91-PAMI} and Yuan~\cite{Yuan1989-ITRA} improved the robustness of the estimates by minimizing the reprojection error. Since the projection function is nonlinear, they used Newton-Raphson method to optimize it. However, both algorithms require good initial values to converge and for both algorithms, each iteration is an $O(n^3)$ operation (requiring the pseudo-inverse of the Jacobian). DeMenthon and Davis proposed an accurate and efficient POSIT algorithm~\cite{DeMenthon95} that iteratively finds object pose by assuming a scaled orthographic projection.

\subsection{A SDM Solution}

The 3D pose estimation problem can also be formulated as a constrained NLS problem that minimizes the reprojection error w.r.t. $\R$ and $\t$: 
\begin{equation*}
\begin{aligned}
\label{eq::3d-pose}
& \underset{\Q,\t}{\text{minimize}}
& & \|\h(\Q,\t,\M) - \U\|_F\\
& \text{subject to} & & \Q^\top\Q = \I_3 \ \text{and} \ \text{det}(\Q) = 1.
\end{aligned}
\end{equation*}
$\h = \g_2 \circ \g_1$ can be seen as composition of two functions $\g_1$ and $\g_2$, which can be written in closed-form as follows:
\begin{align*}
\g_1(\Q,\t,\X) & =  \K(\Q \X + \mathbf{1}_n^\top \otimes \t),\\
\g_2(\X) & =
\begin{bmatrix}
\x_1^\top \oslash \x_3^\top \\
\x_2^\top \oslash \x_3^\top
\end{bmatrix},
\end{align*}
where $\otimes$ represents the Kronecker product, $\oslash$ denotes element-wise division, and $\x_1^\top$ is the first row vector of $\X$.
We parameterize the rotation matrix as a function of the Euler-angles $\btheta$. Then, the objective function can be simplified into the following unconstrained optimization problem:
$$
\min_{\p} \|\h(\p,\M) - \U\|_F,
$$
where $\p=[\btheta;\t]$. We minimize the above function using reversed SDM introduced in Section \ref{sec::reverse-sdm}. For training SDM, we sample a set of poses $\{\p_*^i\}$ and compute the image projections $\{\U^i\}$ under each pose. Recall that the training of reversed SDM alternates between minimization of Eq.~\ref{eq::rsdm-learning} and updating of Eq.~\ref{eq::rsdm-update}. We rewrite these equations in the context of pose estimation:
\begin{equation}
\label{eq::pose-update}
\p_k^i =  \p_{k-1}^i - \R_{k-1} (\h(\p_{k-1}^i,\M)-\U^i),
\end{equation}
$$
\sum_{i}\|\p_*^i - \p^i_{k} + \R_{k} (\h(\p_{k}^i,\M)-\U^i)\|^2_2.
$$
In testing, given an unseen $\tilde{\U}$, SDM recursively applies the update given by Eq.~\ref{eq::pose-update}.

Fig.~\ref{fig::pose} shows three examples of how the reprojection errors are decreased through each SDM update when performing head pose estimation. In these cases, the SDM always starts at $\p_0$ (see iteration 0 in Fig. \ref{fig::pose}) and quickly converges to the optimal solutions. More results can be found in section \ref{sec::exp::pose} as well as a comparison with the POSIT algorithm.

\section{Experiments}

This section illustrates the benefits of SDM for solving NLS problems in a synthetic example and  three computer vision problems.
First, we illustrate SDM with four analytic functions.
Second, we show how we can use SDM to track planar objects that undergo an affine transformation in HoG space using an LK-type tracker.
Third, we demonstrate how SDM can achieve state-of-the-art facial feature detection results in two face-in-the-wild databases. Finally,
in the third experiment we illustrate how SDM can be applied to pose estimation and compare it with POSIT. 
\begin{figure}[t]
\centering
\includegraphics[width=0.9\linewidth]{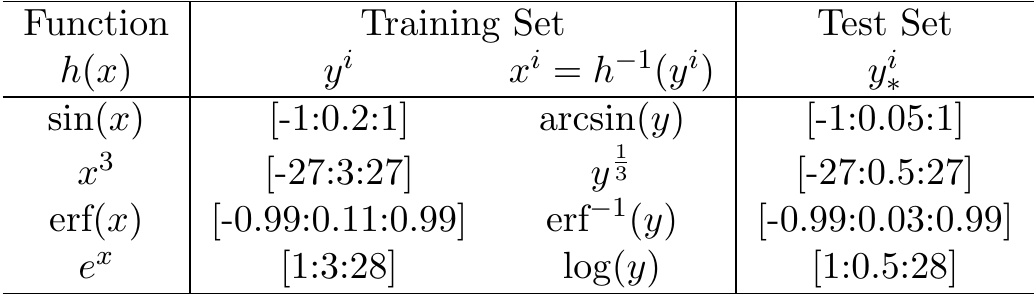}
\caption{Experimental setup for reversed SDM on analytic functions.
$\text{erf}(x)$ is the error function.}
\label{tab::exp::analytic-setup}
\end{figure}
\subsection{Analytic scalar functions}
\label{sec::exp::analytic}

This experiment compares the speed and accuracy performance of reversed SDM against Newton's method on four analytic functions.
The NLS problem that we optimize is:
$$
\min_x f(x) = (h(x)-y_*)^2,
$$
where $h(x)$ is a scalar function (see Fig.~\ref{tab::exp::analytic-setup}) and $y_*$ is a given constant.
Observe that the  $1^{st}$ and $2^{nd}$ derivatives of these functions can be derived analytically.
Assume that we have a fixed initialization $x_0 = c$ and that
we are given a set of training data $\x = \{x^i\}_{i=1}^n$ and $\{h(x^i)\}_{i=1}^n$.
We trained reversed SDM as explained in section~\ref{sec::reverse-sdm}.

The training and testing setup for each function are shown in Fig.~\ref{tab::exp::analytic-setup} using Matlab notation.
We have chosen only invertible functions. Otherwise, for a given $y_*$, multiple solutions may be obtained.
In the training data, the output variables $y$ are sampled uniformly in a local region of $h(x)$,
and their corresponding inputs $x$ are computed by evaluating $y$ at the inverse function of $h(x)$. The test data $\y_* = \{y_*^i\}$ is generated at a finer resolution than in training.

To measure the accuracy of both methods, we computed the normalized least square residuals $\frac{\|\x_k-\x_*\|}{\|\x_*\|}$  at the first 10 steps.
Fig.~\ref{fig::exp::analytic} shows the convergence comparison between  SDM and Newton's method.
Surprisingly, SDM converges with the same number of iterations as Newton's method, but each iteration is faster.
Moreover, SDM is more robust against bad initializations and ill-conditions ($f^{\prime\prime}<0$). For example,
when $h(x)=x^3$, the Newton's method starts from a saddle point and stays there for the following
iterations (observe that in the Fig.~\ref{fig::exp::analytic}b, Newton's method stays at 1).
In the case of $h(x)=e^x$, Newton's method diverges because it is ill-conditioned. Not surprisingly, when Newton's method converges, it provides more accurate estimation than SDM because SDM uses a generic descent map.  If $f$ is quadratic (\eg, $h$ is linear function of $x$), SDM will converge in one iteration because the
average gradient evaluated at different locations will be the same for linear functions.
This coincides with a well-known fact that Newton's method converges in one iteration
for quadratic functions.
\begin{figure}[t]
\centering
\begin{tabular}{cc}
\includegraphics[width=0.48\linewidth]{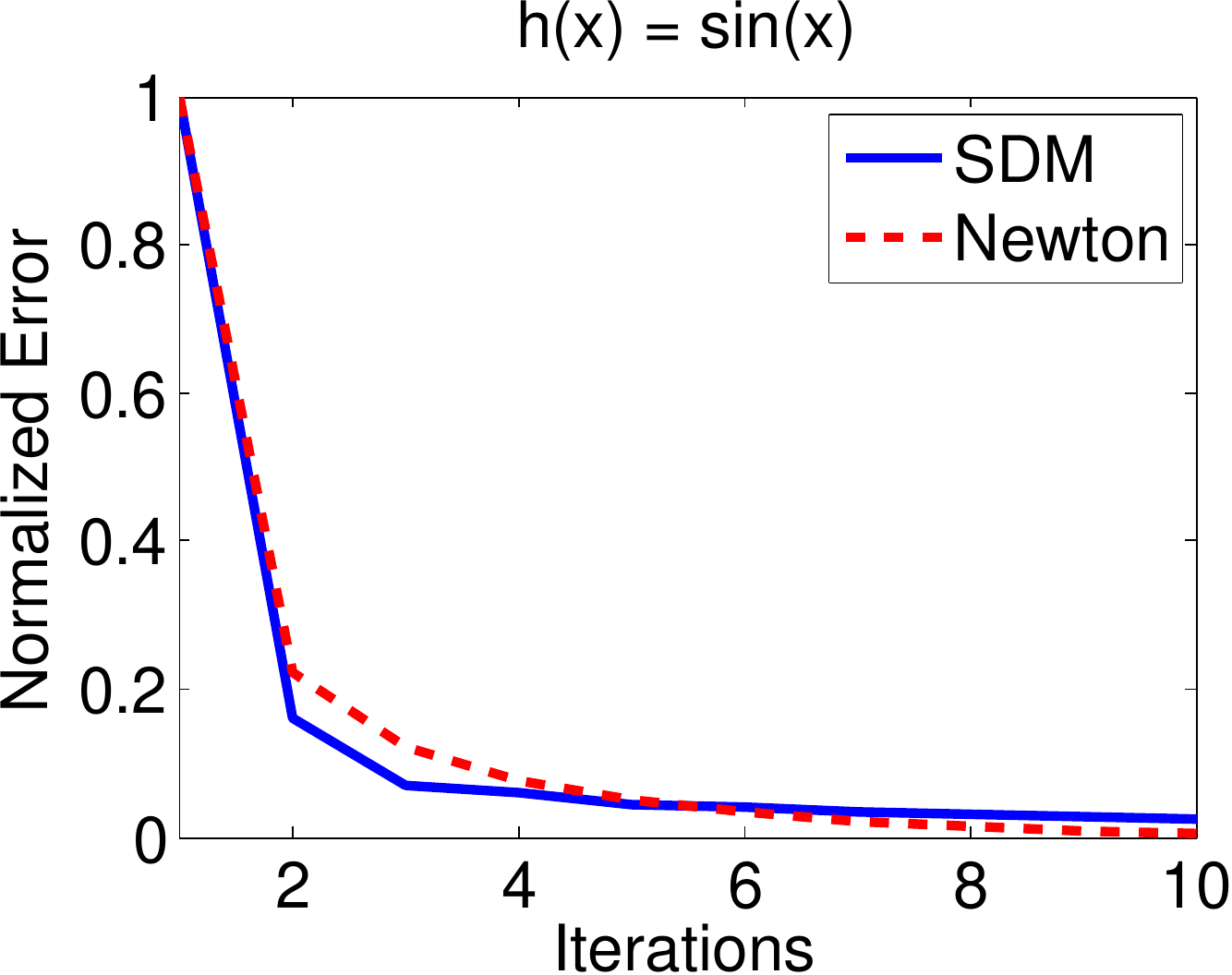}&
\includegraphics[width=0.48\linewidth]{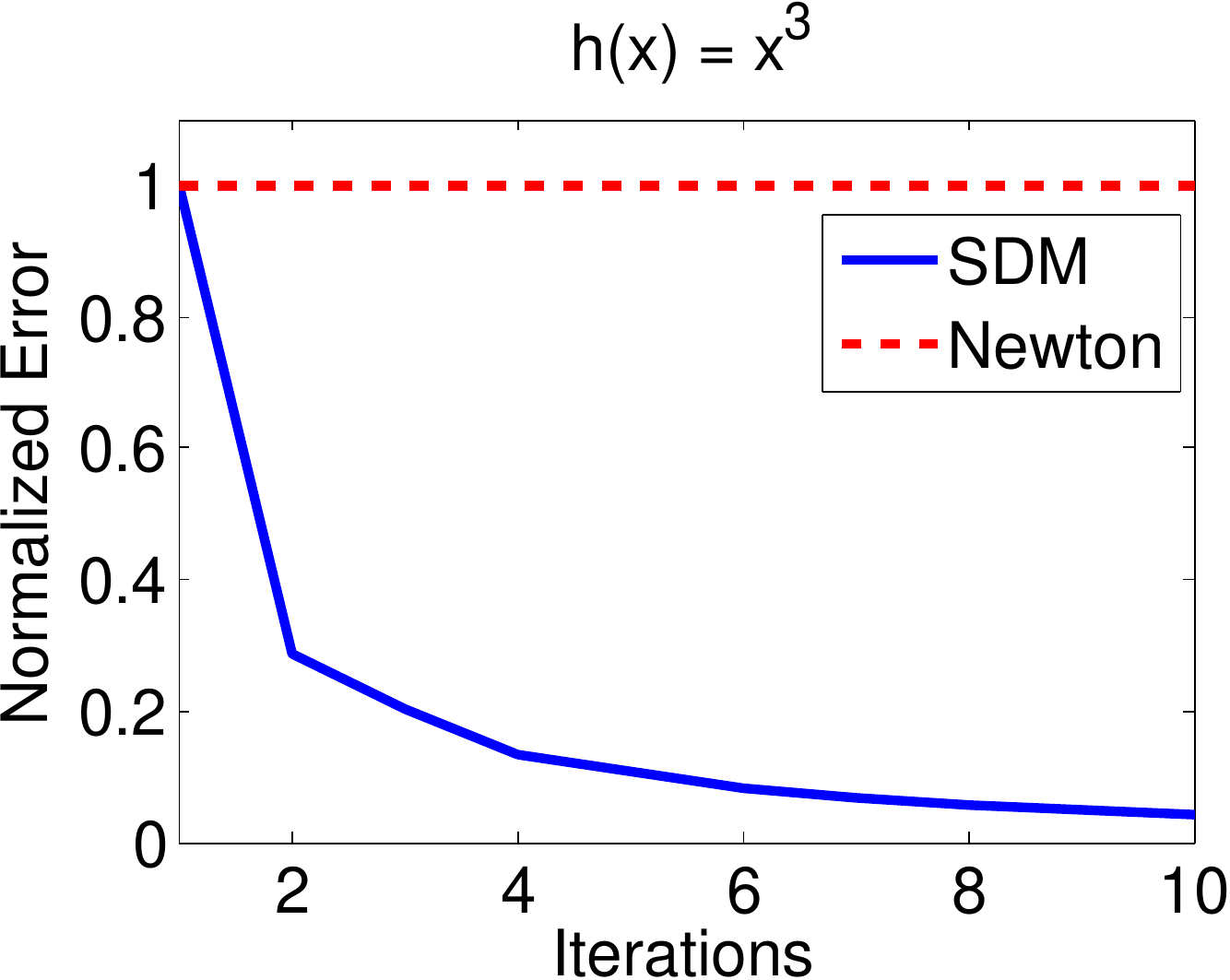}
\\[-2.5pt]
(a) & (b)\\
\includegraphics[width=0.48\linewidth]{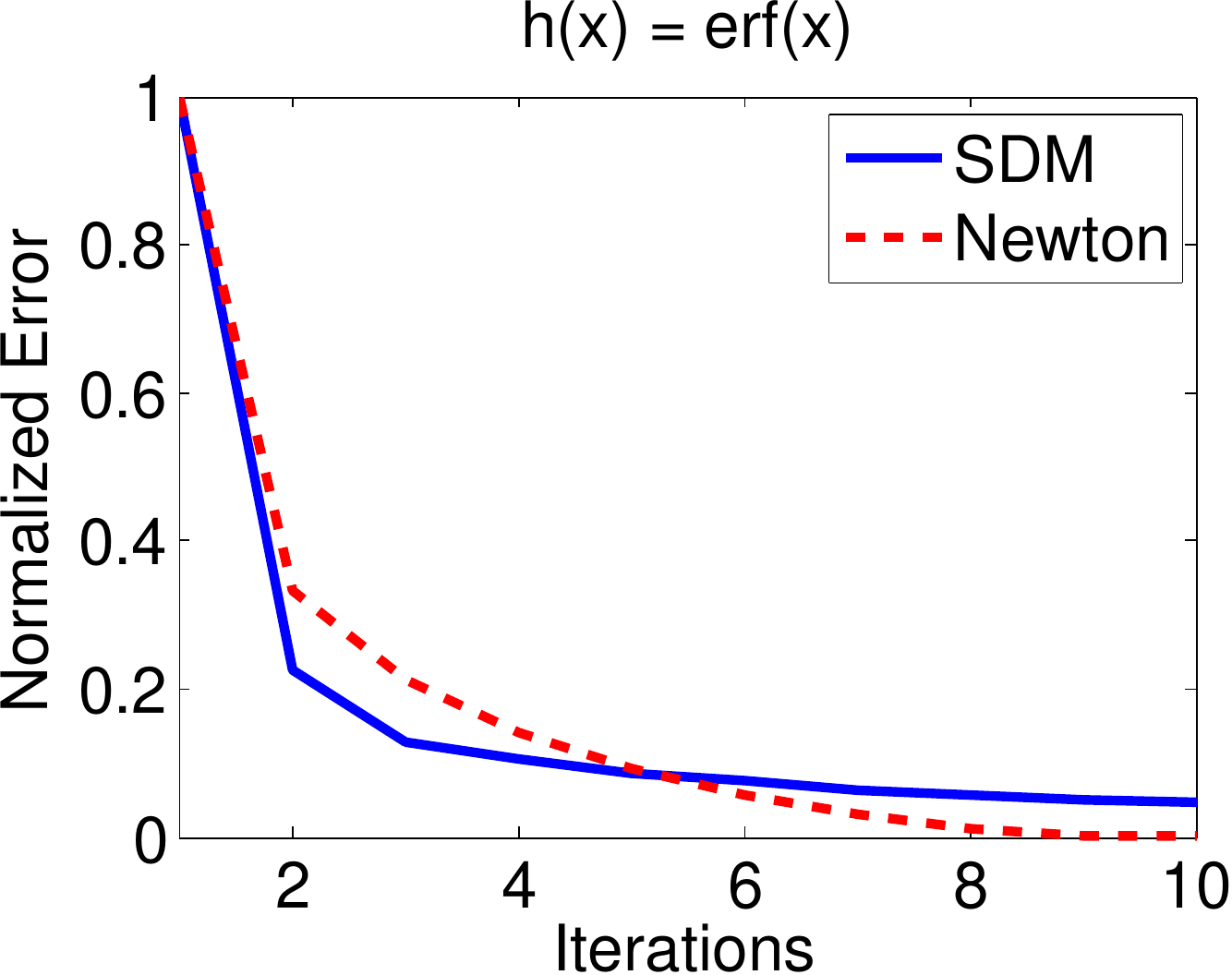}&
\includegraphics[width=0.48\linewidth]{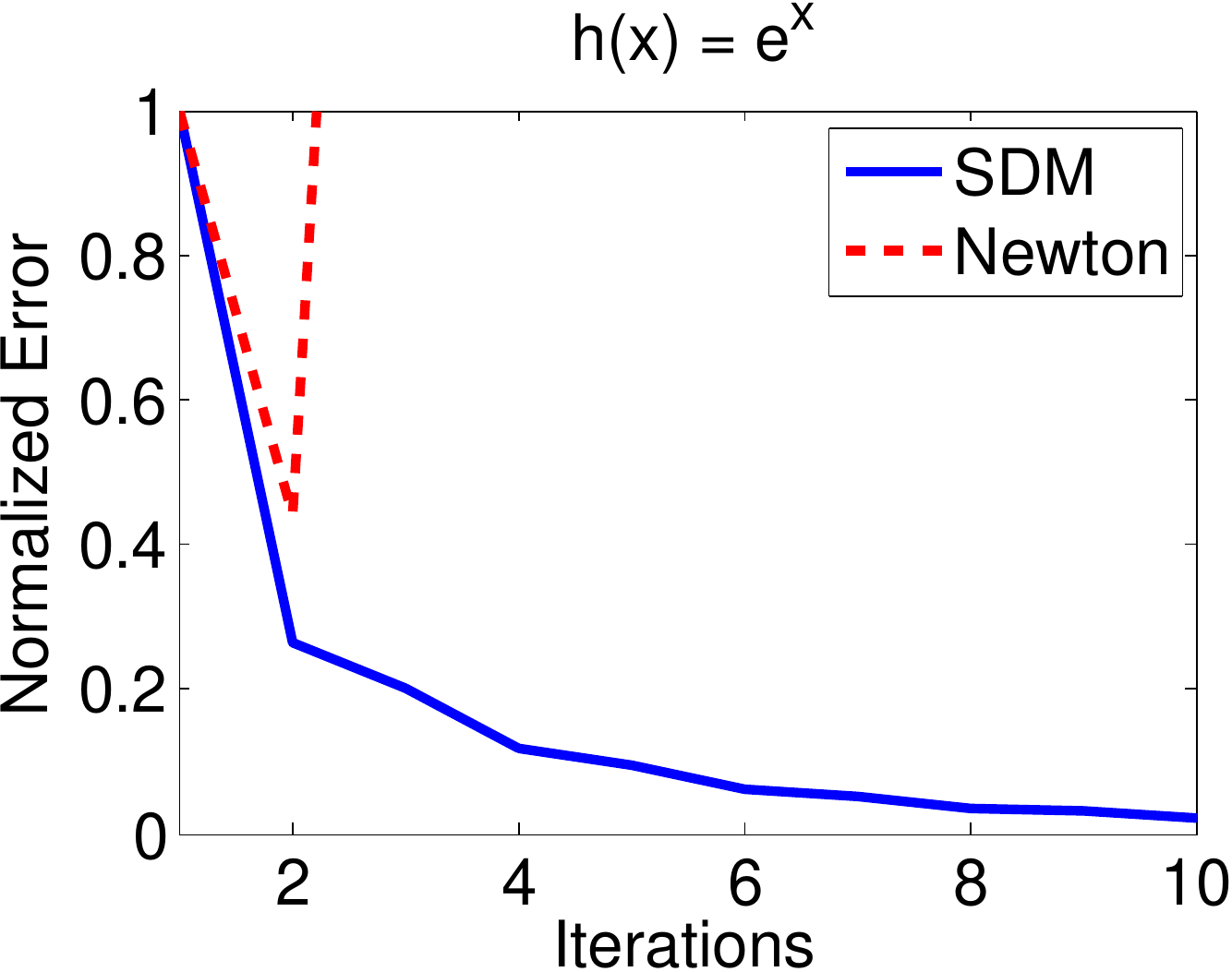}
\\[-2.5pt]
(c) & (d)
\end{tabular}
\vspace{-10pt}
\caption{ Normalized error versus iterations on four analytic functions using Newton's method and SDM.}
\label{fig::exp::analytic}
\end{figure}

\subsection{Rigid tracking}
\label{sec::exp::rigid-tracking}

This section presents the tracking results comparing LK and SDM using an affine transformation.
We used a publicly available implementation of the LK method~\cite{Baker04-ijcv} for tracking a single template. The experiments are conducted on a public dataset\footnote{\url{http://ilab.cs.ucsb.edu/tracking_dataset_ijcv/}} published by~\cite{Gauglitz_IJCV2011}. The dataset features six different planar textures: mansion, sunset, Paris, wood, building, and bricks. Each texture includes $16$ videos, each of which corresponds to a different camera motion path or changes illumination condition. In our experiment, we chose five of the $16$ motions, giving us a total of $30$ videos. The five motions correspond to translation, dynamic lighting, in-plane rotation, out-plane rotation, and scaling. 

Both trackers (SDM and LK) used the same template, which was extracted at a local region on the texture in the first frame. In our implementation of SDM, the motion parameters were sampled from an isotropic Gaussian distribution with zero mean. The standard deviations were set to be
$[0.05, 0.05, 0.05, 0.05, 8, 8]^\top.$ We used 300 samples and four iterations to train SDM.
The tracker was considered lost if there was more than $30\%$ difference between the template and the back-warp image. Note that this difference was computed in HoG space.

Table \ref{tab::exp::rigid-tracking} shows the number of frames successfully tracked by the LK tracker and SDM. SDM performs better than or as well as the LK tracker in $28$ out of the $30$ sequences. We observe that SDM performs much better than LK in translation. One possible explanation is that HoG features are more robust to motion blur. Not surprisingly, SDM performs perfectly in the presence of dynamic lighting because HoG is robust to illumination changes.
In-plane rotation tends to be the most challenging motion for SDM, but even in this case, it is very similar to LK.
\begin{figure}[t]
\label{tab::exp::rigid-tracking}
\centering
\includegraphics[width=1.\linewidth]{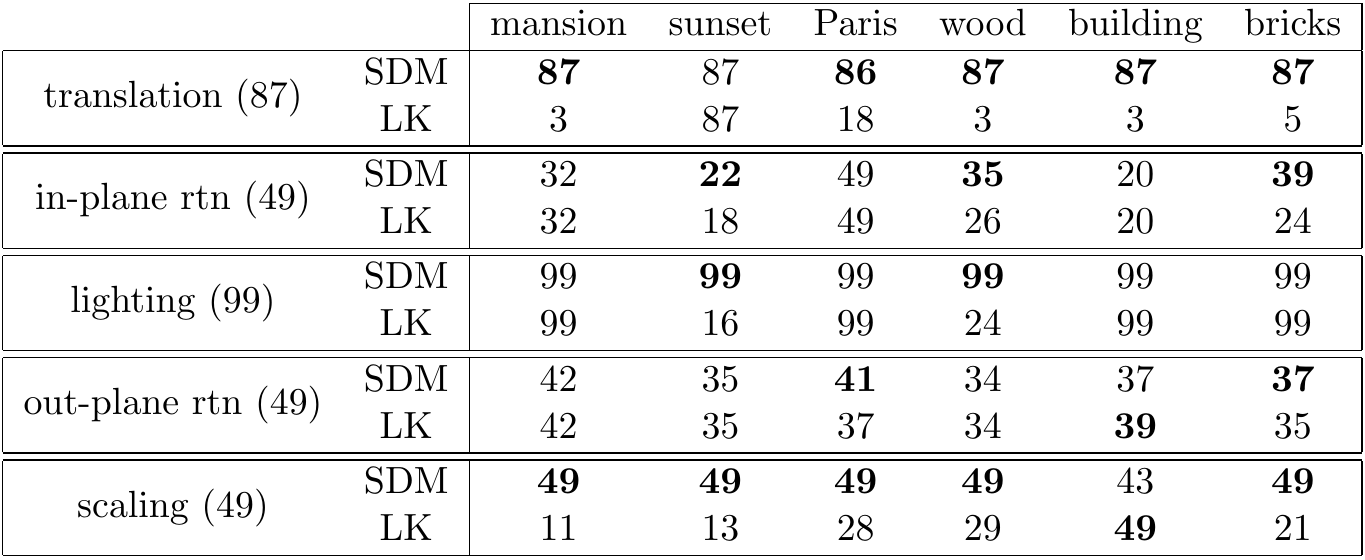}
\caption{Comparison between SDM and LK on rigid tracking experiments. Each entry in the table states the number of frames successfully tracked by each algorithm. The total number of frames is given by number in the parentheses from the first column.}
\end{figure}

\subsection{ Facial feature detection}
\label{sec::exp::detection}
This section reports experiments on facial feature detection in two ``face-in-the-wild" datasets, and compares SDM with state-of-the-art methods.
The two face databases are the LFPW dataset\footnote{\url{http://www.kbvt.com/LFPW/}}~\cite{Belhumeur11} and
the LFW-A\&C dataset~\cite{Saragih-CVPR11}.


The experimental setup is as follows. First, the face is detected using the OpenCV face detector~\cite{opencv_library}.
The evaluation is performed on those images in which a face can be detected.
The face detection rates are 96.7\% on LFPW and 98.7\% on LFW-A\&C, respectively.
The initial shape estimate is given by centering the mean face at the normalized square. The translational and scaling differences between the initial and true landmark locations are also computed, and their means and variances are used for generating Monte Carlo samples in Eq.~\ref{eq::monte-carlo}. We generated 10 perturbed samples for each training image. HoG descriptors are computed on $32\times32$ local patches around each landmark. To reduce the dimensionality of the data, we performed PCA, preserving $98\%$ of the energy on the image features.


{\bf LFPW} dataset contains images downloaded from the web that exhibit
large variations in pose, illumination, and facial expression.
Unfortunately, only image URLs are given and some are no longer valid.
We downloaded 884 of the 1132 training images and 245 of the 300
test images. We followed the evaluation metric used in~\cite{Belhumeur11}, where
the error is measured as the average Euclidean distance between
the 29 labeled and predicted landmarks. The error is then normalized by the
inter-ocular distance.

We compared our approach with two recently proposed methods~\cite{Belhumeur11,Cao12}.
Fig.~\ref{fig::cde} shows the Cumulative Error Distribution (CED) curves of SDM,
Belhumeur \etal~\cite{Belhumeur11}, and our method trained with only one linear regression. Note that SDM
is different from the AAM trained in a discriminative manner with linear regression~\cite{Cootes-et-al-PAMI01} because we do not learn a shape or appearance model.
Note that such curves are computed from $17$ of the $29$ points defined in~\cite{Cristinacce08}, following the convention used in~\cite{Belhumeur11}. Clearly, SDM outperforms~\cite{Belhumeur11} and linear regression. It is also important to notice that a
completely fair comparison is not possible since~\cite{Belhumeur11} was trained
and tested with some images that were no longer available. However, the
average is in favor of our method. The recently proposed method in~\cite{Cao12} is based on boosted regression with pose-indexed features.
To the best of our knowledge this paper reported the state-of-the-art results on LFPW dataset.
In~\cite{Cao12}, no CED curve was given and they reported a mean error ($\times 10^{-2}$) of 3.43.
SDM shows comparable performance with a average of $3.47$.

The first two rows of Fig.~\ref{fig::lfpw} show our results on faces with large variations
in poses and illumination as well as ones that are partially occluded.
The last row displays the {\em worst} 10 results measured by the normalized mean error.
Most errors were caused by the gradient feature's inability
to distinguish between similar facial parts and occluding objects
(\eg, glasses frame and eye brows).

{\bf LFW-A\&C} is a subset of the LFW dataset\footnote{\url{http://vis-www.cs.umass.edu/lfw/}},
consisting of 1116 images of people whose names begin with an 'A' or 'C'.
Each image is annotated with the same 66 landmarks shown in Fig. \ref{fig::init}.
We compared our method with the Principle Regression Analysis (PRA) method~\cite{Saragih-CVPR11},
which proposes a sample-specific prior to constrain the regression output.
This method achieves the state-of-the-art results on this dataset.
Following~\cite{Saragih-CVPR11}, those whose name started with `A' were used for training, giving
us a total of 604 images. The remaining images were used for testing.
Root mean squared error (RMSE) was used to measure the alignment accuracy.
Each image has a fixed size of $250\times250$ and the error was not normalized.
PRA reported a median alignment error of 2.8 on the test set while ours averages 2.7.
The comparison of CED curves can be found in Fig.~\ref{fig::cde}b and our method
outperforms both PRA and Linear Regression.
Qualitative results from SDM on the more challenging samples are plotted in
Fig.~\ref{fig::lfw}.

\begin{figure}[t]
\centering
\addtolength{\tabcolsep}{-2pt}
\begin{tabular}{cc}
\includegraphics[width=0.48\linewidth]{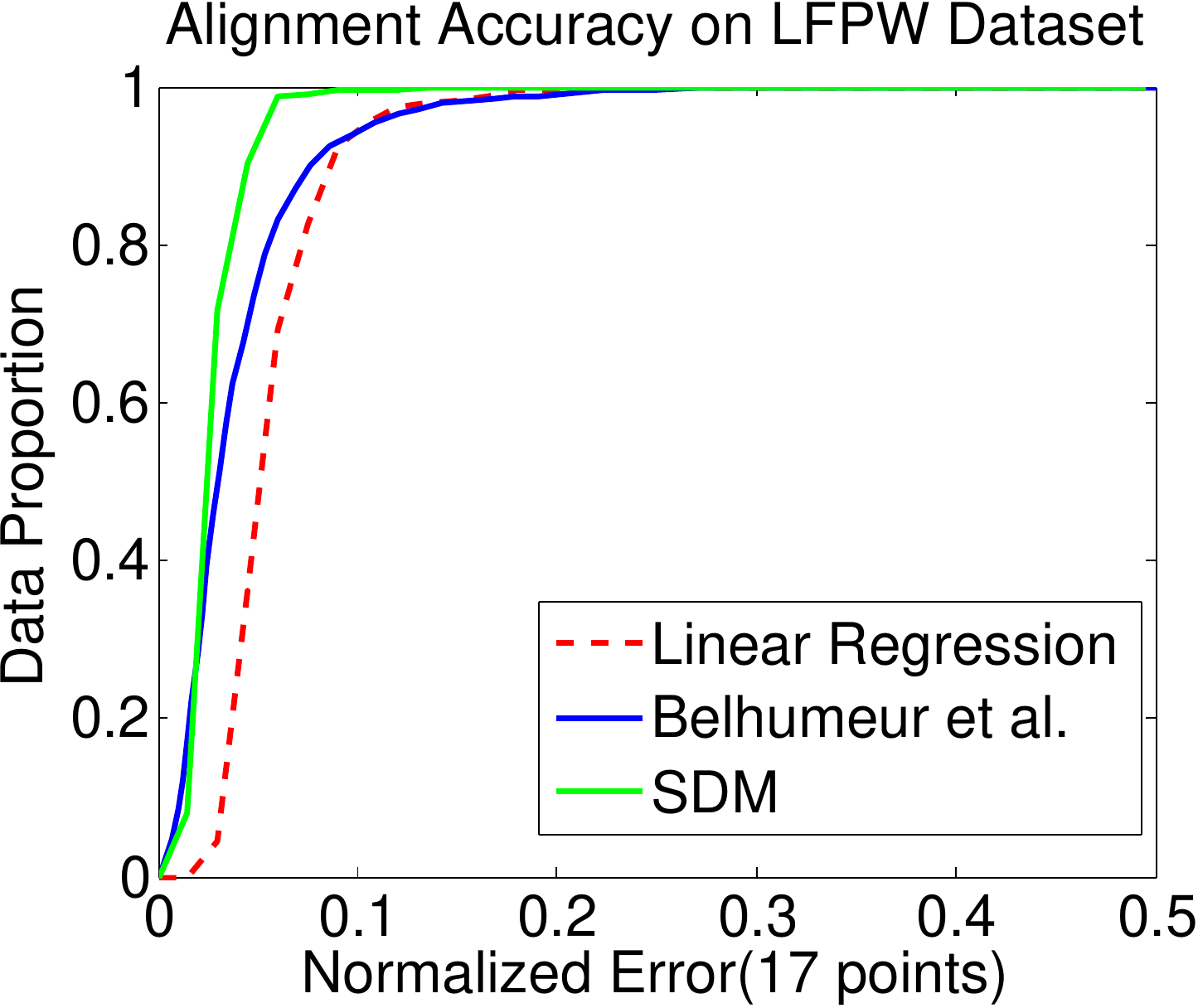}&
\includegraphics[width=0.48\linewidth]{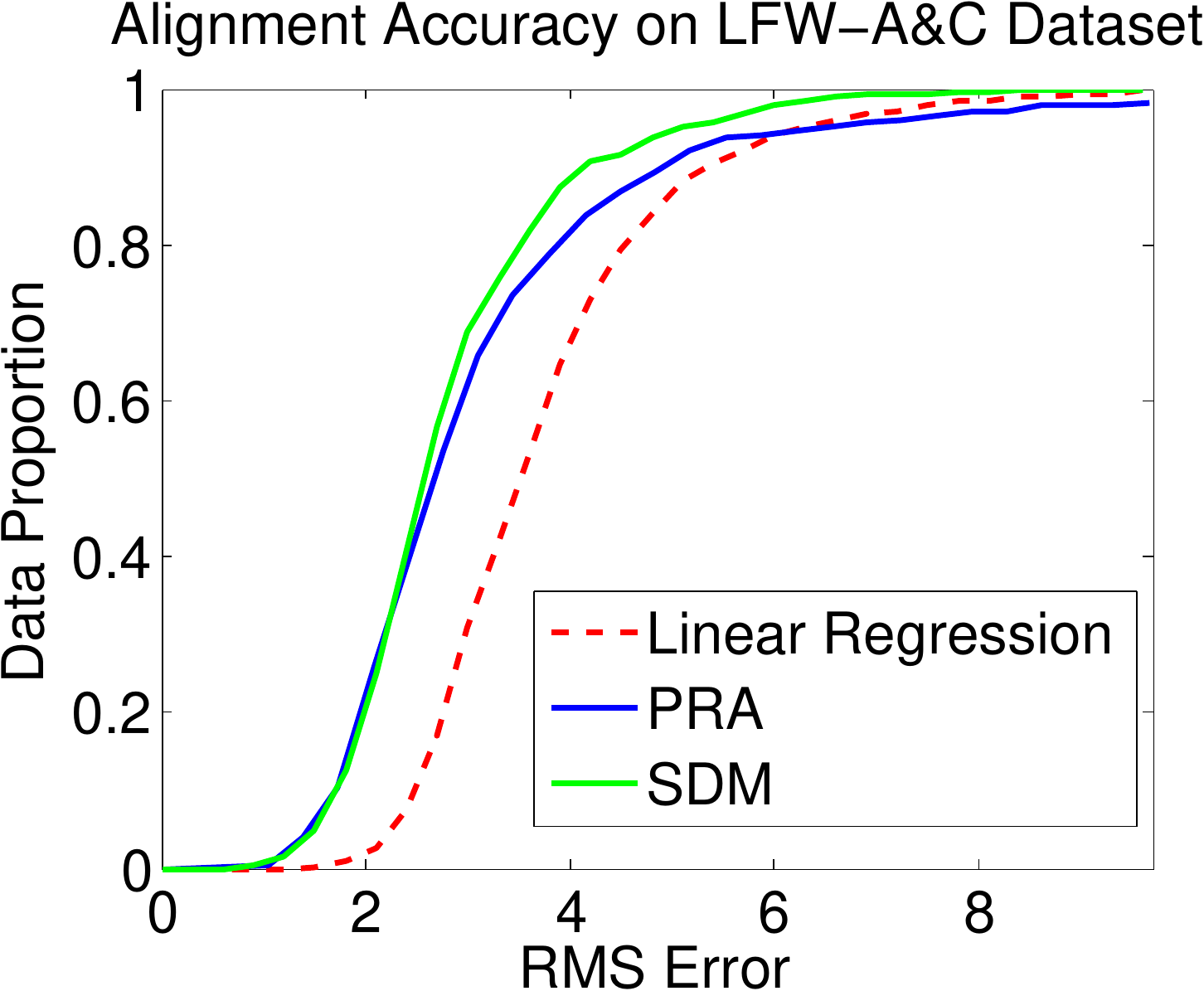}\\[-2.5pt]
(a)&(b)
\end{tabular}
\caption{CED curves from LFPW and LFW-A\&C datasets.}
\label{fig::cde}
\end{figure}
\subsection{Facial feature tracking}

This section tested the use of SDM for facial feature tracking. The main idea is to use
SDM for detection in each frame, but initialize the frame with the landmark estimates of the previous frame.

We trained our model with $66$ landmarks on the MPIE~\cite{Gross-et-al-MPIE07} and LFW-A\&C datasets. The standard deviations of the scaling and translational perturbation were set to 0.05 and 10, respectively. This indicates that in two consecutive frames, the probability of a tracked face shifting more than 20 pixels or scaling more than $10\%$ is less than $5\%$. We evaluated SDM's tracking performance on two datasets: RU-FACS~\cite{nsf_database} and  Youtube Celebrities~\cite{Kim08}.

{\bf RU-FACS} dataset consists of $29$ sequences of different subjects recorded
in a constrained environment. Each sequence has an average of $6300$ frames.
The dataset is labeled with the same 66 landmarks of our trained model except for
the 17 jaw points, which are defined slightly differently (See Fig.~\ref{fig::ru-facs-acc}b). We used the remaining $49$ landmarks for evaluation.
The ground truth was given by person-specific AAMs~\cite{Matthews-Baker-IJCV04}.
For each of the $29$ sequences, the average RMS error and standard deviation
are plotted in red in Fig. \ref{fig::ru-facs-acc}a. The blue lines were generated by a new model re-trained with the first image of each of the 29 sequences following the online SDM algorithm introduced in section~\ref{sec::online-sdm}. The tracking results improved for {\em all} sequences.
To make sense of the numerical results, in Fig.~\ref{fig::ru-facs-acc}b, we also showed one tracking result overlaid with ground truth and in this example, it gives us a RMS error of $5.03$.
There are no obvious differences between the two labelings.
Also, the person-specific AAM gives unreliable results when the subject's face is partially occluded, while SDM still provides a robust estimation (See Fig.~\ref{fig::ru-facs}). In the $170,787$ frames of the RU-FACS videos, the SDM tracker {\em never} lost track, even in cases of partial occlusion.

\begin{figure}[t]
\centering
\addtolength{\tabcolsep}{-2pt}
\begin{tabular}{cc}
\includegraphics[width=0.6\linewidth]{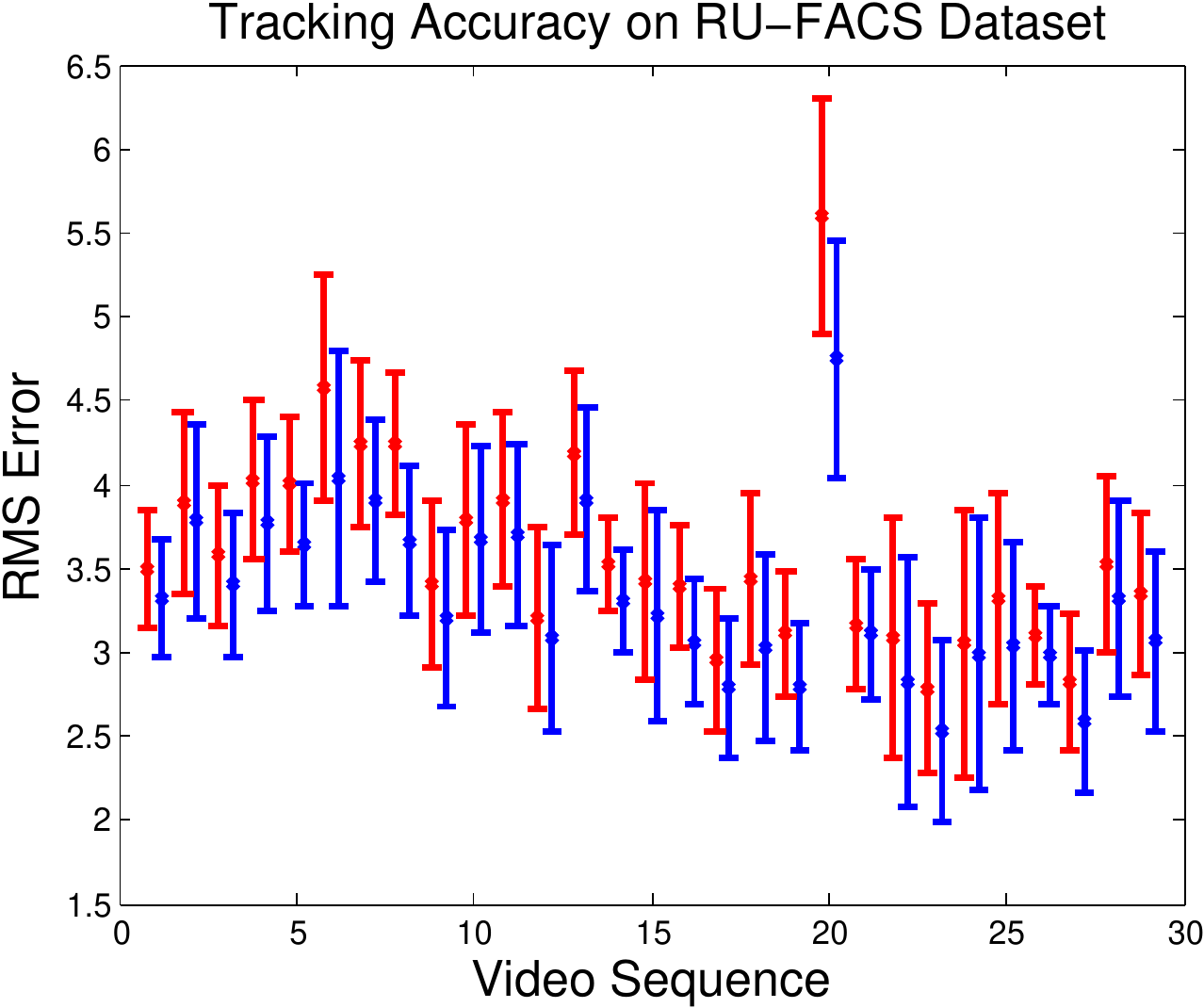}&
\raisebox{6.5pt}{\includegraphics[width=0.35\linewidth]{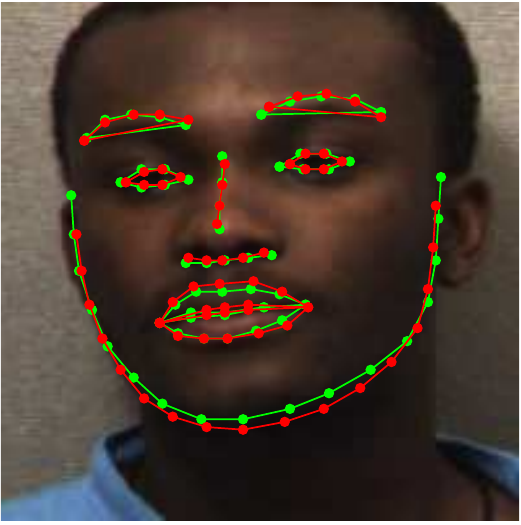}}\\[-2.5pt]
(a)&(b)
\end{tabular}
\caption{a) Average RMS errors and standard deviations on 29 video sequences in RU-FACS dataset. Red: orignal model. Blue: re-trained model.
b) RMS error between the SDM detection (green) and ground truth (red) is $5.03$.}
\label{fig::ru-facs-acc}
\end{figure}

{\bf Youtube Celebrities} is a public ``in-the-wild''
dataset\footnote{\url{www.seqam.rutgers.edu/site/media/data_files/ytcelebrity.tar}}
that is composed of videos of celebrities during an interview or on a TV show.
It contains $1910$ sequences of $47$ subjects, but most of them are shorter than 3 seconds.
It was released as a dataset for face tracking and recognition, thus no labeled facial
landmarks are given. See Fig.~\ref{fig::youtube} for example tracking results from this dataset. Tracked video sequences can be found below\footnote{\url{http://www.youtube.com/user/xiong828/videos}}.
From the videos, we can see that SDM is able to reliably track facial landmarks
with large pose variation ($\pm45^\circ$ yaw, $\pm90^\circ$ roll and, $\pm30^\circ$ pitch), occlusion, and illumination changes. All results are generated without re-initialization.

Both Matlab and C++ implementations of facial feature detection and tracking can be found at the link below\footnote{\url{http://www.humansensing.cs.cmu.edu/intraface}}. The C++ implementation averages around 11ms per frame, tested with an Intel i7 3752M processor. On a mobile platform (iPhone 5s), the tracker runs at 35fps.

\subsection{3D pose estimation}
\label{sec::exp::pose}

This section reports the experimental results on 3D pose estimation using SDM and a comparison with the widely popular POSIT method~\cite{DeMenthon95}.

\begin{figure}[t]
\centering
\begin{tabular}{cc}
\includegraphics[width=0.35\linewidth]{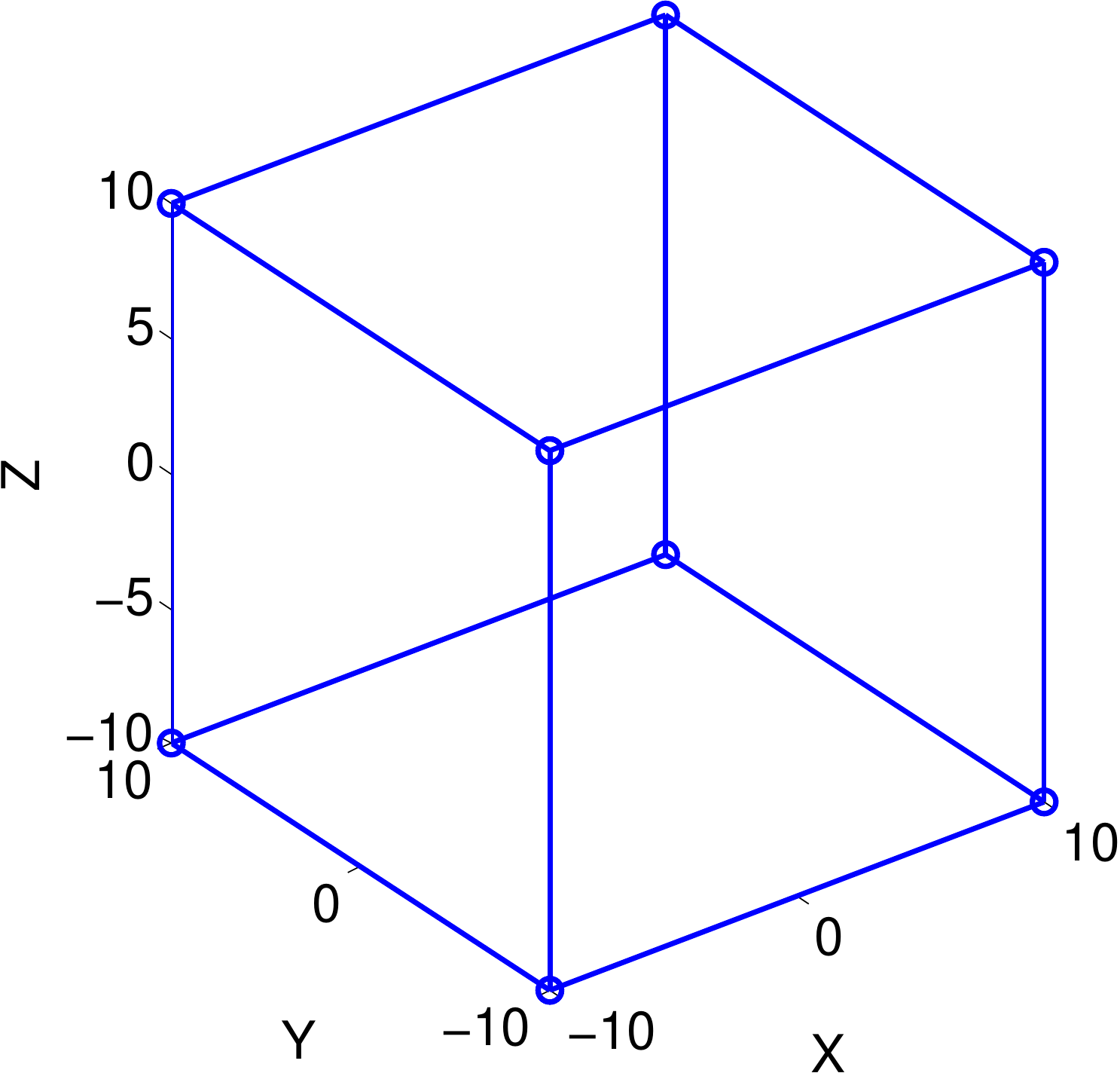}&
\includegraphics[width=0.35\linewidth]{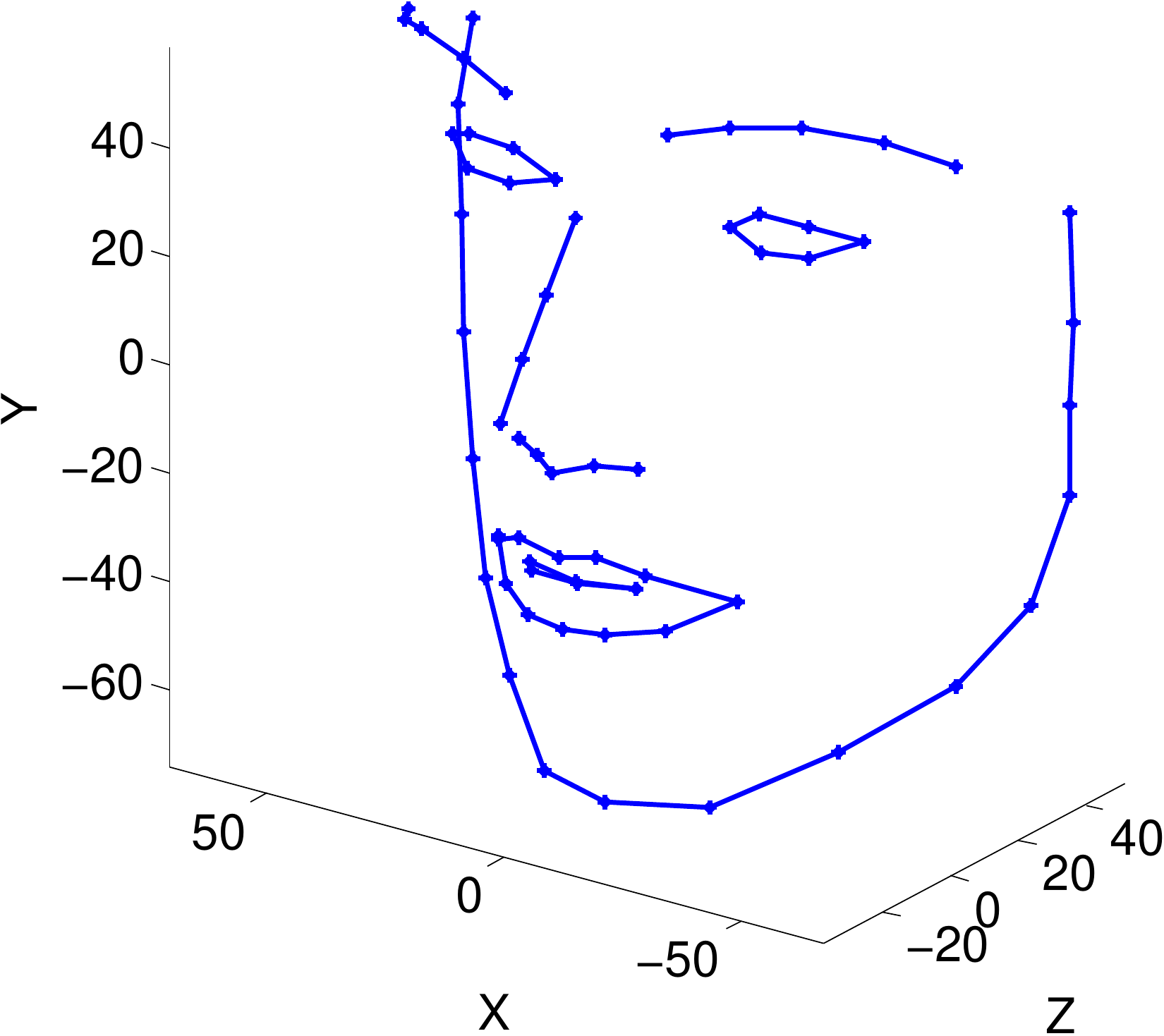}\\
\multicolumn{2}{c}{
\includegraphics[width=0.65\linewidth]{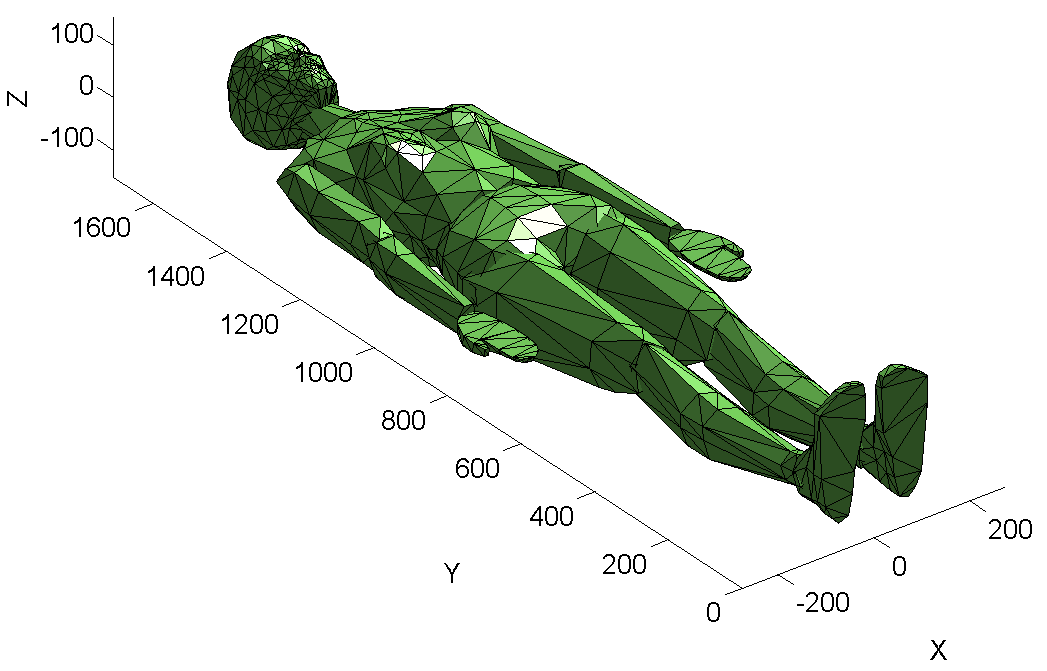}}
\end{tabular}
\caption{3D objects used in pose estimation experiments. Units are in millimeters (mm).}
\label{fig::exp::pose-objects}
\end{figure}
The experiment is set up as follows. We selected three different meshes of $3$D objects: a cube, a face, and a human body\footnote{\url{www.robots.ox.ac.uk/~wmayol/3D/nancy_matlab.html}} (see Fig. \ref{fig::exp::pose-objects}).
We placed a virtual camera at the origin of the world coordinates. In this experiment, we set the focal length (in terms of pixels) to be $f_x = f_y = 1000$ and principle point to be $[u_0, v_0] = [500, 500]$. The skew coefficient was set to be zero. The training and testing data were generated by placing a 3D object at $[0,0,2000]$, perturbed with different 3D translations and rotations. The POSIT algorithm does not require labeled data. Three rotation angles were uniformly sampled from $-30^\circ$ to $30^\circ$ with increments of $10^\circ$ in training and  $7^\circ$ in testing. Three translation values were uniformly sampled from -400mm to 400mm with increments of 200mm in training and 170mm in testing. Then, for each combination of the six values, we computed the object's image projection using the above virtual camera and used it as the input for both algorithms. White noise ($\sigma^2 = 4$) was added to the projected points.  In our implementation of SDM, to ensure numerical stability, the image coordinates $[u,v]$ of the projection were normalized as follows:
$
\begin{bmatrix}
\hat{u}\\
\hat{v}
\end{bmatrix}
=
\begin{bmatrix}
(u-u_0)/f_x\\
(v-v_0)/f_y
\end{bmatrix}.
$

Fig.~\ref{fig::exp::pose}a shows the mean errors and standard deviations of the estimated rotations (in degree) and translations (in mm) for both algorithms. Both algorithms achieved around $1^\circ$ accuracy for rotation estimation, but SDM was more accurate for translation. This is because POSIT assumes a scaled orthographic projection, while the true image points are generated by a perspective projection. Fig.~\ref{fig::exp::pose}b states the computational time in milliseconds (ms) for both methods. Both algorithms are efficient, although POSIT is slightly faster than SDM.

\begin{figure*}[t]
\centering
\addtolength{\tabcolsep}{-3pt}
\begin{tabular}{cc}
\includegraphics[width=0.70\linewidth]{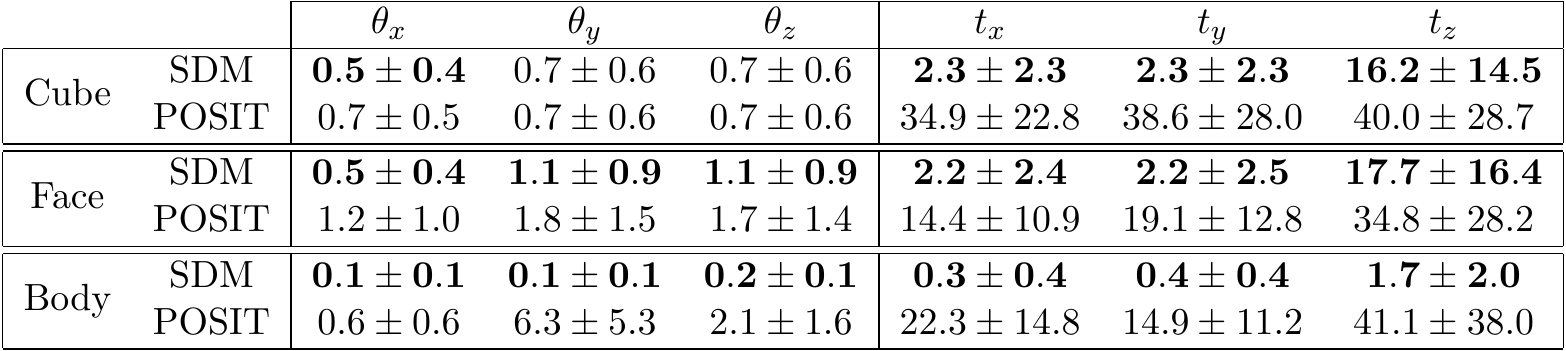}&
\includegraphics[width=0.28\linewidth]{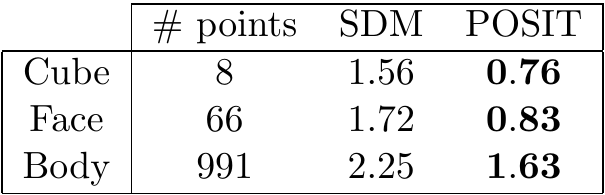}\\[-2.5pt]
(a)&(b)
\end{tabular}
\caption{Accuracy and speed comparison between SDM and POSIT algorithms on estimating 3D object pose. a) Rotation (in degree) and translation (in mm) errors and their standard deviations.
b) Running times (in ms) of both algorithms.}
\label{fig::exp::pose}
\end{figure*}

\section{Conclusions}

This paper presents SDM, a method for solving NLS problems. SDM learns generic DMs in a supervised manner, and is able to overcome many drawbacks of second order optimization schemes, such as non-differentiability and expensive computation of the Jacobians and Hessians. Moreover, it is extremely fast and accurate. We have illustrated the benefits of our approach in three important computer vision problems and achieve state-of-the-art performance
in the problem of facial feature detection.

Beyond SDM, an important contribution of this work in the context discriminative face alignment is to propose a cost function for discriminative
image alignment (Eq.~\ref{eq::face-alignment}).  Existing discriminative methods for facial alignment pose the problem as a regression one, but lack a well-defined alignment error function. Therefore, performing theoretical analysis of these methods is difficult. Eq.~\ref{eq::face-alignment} establishes a direct connection with existing PAMs, and existing algorithms can be used to minimize it, such as Gauss-Newton or the supervised version proposed in this paper. In addition, we were able to use robust feature representations (\eg, HoG) within commonly used alignment algorithms, greatly improving the robustness and performance. 

While SDM has shown to be a promising alternative to second order methods of solving NLS problems, according  to Theorem~\ref{th::high-dim}, the generic DM only exists within a local neighborhood of the optimal parameters. When we train SDM with a large range of parameters, the results may not be accurate. 
For instance, in the experiment of pose estimation, when we trained SDM with a larger range of rotations (\eg, $-60^\circ$ to $60^\circ$), the performance dropped dramatically. 
A simple way to solve this problem is to partition the output space into grids and train a different SDM for each grid. During testing, we can iterate over all the models and return the solution that give us the minimum reprojection error. We will explore more about how to use SDM for global optimization in future work.

\section{Appendix}
\label{sec::app}
\begin{theorem}
\label{th::one-dim}
If the function $h(x)$ satisfies the following two conditions:
\begin{enumerate}
\item $h(x)$ is monotonic around the minimum $x_*$,
\item $h(x)$ is locally Lipschitz continuous anchored at $x_*$ with $K$ as the Lipschitz constant,
\end{enumerate}
then $r$ is a generic DM if $|r| < \frac{2}{K}$ and $\text{sign}(r) = \text{sign}(h^\prime(x))$.
\end{theorem}

\begin{IEEEproof}
Without loss of generality, we assume that $h(x)$ is monotonically increasing, and that $h(x_k) \neq h(x_*)$. Otherwise, the optimization has reached the minimum.
In the following, we use $\Delta x_k^i$ to denote $x_* - x^i_k$ and $\Delta h^i_k$ to denote $h(x_*)-h(x^i_k)$. We want to find $r$ such that
\begin{equation}
\label{eq::proof1}
\frac{|\Delta x^i_{k}|}{|\Delta x^i_{k-1}|}  < 1, \text{if } x_* \neq x_{k-1}^i.
\end{equation}
We replace $x_{k}^i$ with $x_{k-1}^i$ using Eq.~\ref{eq::update-1d} and the left side of Eq.~\ref{eq::proof1} becomes
\begin{align}
 \frac{|\Delta x^i_{k}|}{|\Delta x^i_{k-1}|}
 &=\frac{|\Delta x^i_{k-1}-r\Delta h_{k-1}^i|}{|\Delta x^i_{k-1}|} =\frac{|\Delta x^i_{k-1}(1-r\frac{\Delta h_{k-1}^i}{\Delta x^i_{k-1}})|}{|\Delta x^i_{k-1}|} \nonumber \\
& =\frac{|\Delta x^i_{k-1}||1-r\frac{\Delta h_{k-1}^i}{\Delta x^i_{k-1}}|}{|\Delta x^i_{k-1}|} =\left|1-r\frac{\Delta h_{k-1}^i}{\Delta x^i_{k-1}}\right|\nonumber\\
\label{eq::proof2}
& =\left|1-r\frac{|\Delta h_{k-1}^i|}{|\Delta x^i_{k-1}|}\right|.
\end{align}
The last step is derived from condition 1. Denoting $\frac{|\Delta h_{k-1}^i|}{|\Delta x^i_{k-1}|}$ as $K_{k-1}^i$ and setting Eq.~\ref{eq::proof2} $ < 1$ gives us
\begin{alignat}{3}
            -1 &<{}& 1 - &rK_{k-1}^i{}& &< 1 \nonumber\\
\label{eq::proof3}
\Rightarrow 0  &<{}&  &r{}&             &< \frac{2}{K_{k-1}^i}.
\end{alignat}
From condition 2, we know that $K^{i}_{k-1} \leq K$. Any $0<r<\frac{2}{K}$ will satisfy the inequalities in Eq.~\ref{eq::proof3}, and therefore, there exists a generic DM. Similarly, we can show $0>r>-\frac{2}{K}$ is a generic DM when $h(x)$ is a monotonically decreasing.
\end{IEEEproof}

\begin{theorem}
\label{th::high-dim}
If function $\h(\x)$ satisfies the following two conditions:
\begin{enumerate}
\item $\g(\x) = \R\h(\x)$ is a locally monotone operator anchored at the minimum $\x_*$,
\item $\h(\x)$ is locally Lipschitz continuous anchored at $\x_*$ with $K$ as the Lipschitz constant,
\end{enumerate}
then there exists a generic DM $\R$.
\end{theorem}
\begin{proof}
To simplify the notation, we denote $\x_* - \x$ as $\Delta\x$, $\h(\x_*)-\h(\x)$ as $\Delta\h$, and use $\|\x\|$ to represent the L-2 norm. We want to show that there exists $\R$ such that
\begin{equation}
\label{eq::proof4}
\frac{\|\x_{*}-\x^i_{k}\|}{\|\x_{*}-\x^i_{k-1}\|}  < 1, \text{if } \x_* \neq \x_{k-1}^i.
\end{equation}
We replace $\x_k^i$ with $\x_{k-1}^i$ using Eq.~\ref{eq::update-hd} and squaring the left side of Eq.~\ref{eq::proof4} gives us
\begin{align}
& \frac{\|\Delta\x^i_k\|^2}{\|\Delta\x^i_{k-1}\|^2}
 = \frac{\|\Delta\x^i_{k-1}-\R\Delta\h_{k-1}^i\|^2}{\|\Delta\x^i_{k-1}\|^2}\nonumber\\
 = & \frac{\|\Delta\x^i_{k-1}\|^2}{\|\Delta\x^i_{k-1}\|^2} + \frac{\|\R\Delta\h_{k-1}^i\|^2}{\|\Delta\x^i_{k-1}\|^2}
-2\frac{\Delta\x_{k-1}^{i^\top}\R\Delta\h_{k-1}^i}{\|\Delta\x^i_{k-1}\|^2}\nonumber\\
 = & 1 + \frac{\|\R\Delta\h_{k-1}^i\|}{\|\Delta\x^i_{k-1}\|^2}
\bigg(\|\R\Delta\h_{k-1}^i\|
\label{eq::proof5}
-2\Delta\x_{k-1}^{i^\top}\frac{\R\Delta\h_{k-1}^i}{\|\R\Delta\h_{k-1}^i\|}\bigg).
\end{align}
Setting Eq.~\ref{eq::proof5} $< 1$ gives us,
\begin{equation}
\label{eq::proof6}
\|\R\Delta\h_{k-1}^i\| \leq 2\Delta\x_{k-1}^{i^\top} \frac{\R\Delta\h_{k-1}^i}{\|\R\Delta\h_{k-1}^i\|}
\end{equation}
Condition 1 ensures that $\Delta\x_{k-1}^{i^\top} \R\Delta\h_{k-1}^i > 0$. From the geometric definition of dot product, we can rewrite the right side of the inequality~\ref{eq::proof6}  as,
\begin{equation*}
2\Delta\x_{k-1}^{i^\top} \frac{\R\Delta\h_{k-1}^i}{\|\R\Delta\h_{k-1}^i\|} =
2\|\Delta\x_{k-1}^i\| \cos\theta^i,
\end{equation*}
where $\theta^i$ is the angle between vectors $\Delta\x_{k-1}^i$ and $\R\Delta\h_{k-1}^i$. Using condition 2 we have
\begin{equation}
\label{eq::proof7}
2\|\Delta\x_{k-1}^i\| \cos\theta^i \geq \frac{2}{K} \|\Delta\h_{k-1}^i\| \cos\theta^i
\end{equation}
From the Cauchy-Schwartz inequality,
\begin{equation}
\label{eq::proof8}
\|\R\Delta\h_{k-1}^i\| \leq \|\R\|_F\|\Delta\h_{k-1}^i\|.
\end{equation}
Given the inequalities in Eqs.~\ref{eq::proof7} and~\ref{eq::proof8}, the condition that makes Eq.~\ref{eq::proof6} hold is, 
\begin{equation}
\|\R\|_F \leq \frac{2}{K}\cos\theta^i.
\end{equation}
Any $\R$ whose $\|\R\|_F < \frac{2}{K}\min_i\cos\theta^i$ gaurantees the inequality stated in Eq.~\ref{eq::proof4}. Therefore, there exists a generic DM.
\end{proof}


%


\ifCLASSOPTIONcompsoc
\else
\fi


\ifCLASSOPTIONcaptionsoff
  \newpage
\fi



\bibliographystyle{IEEEtran}
\bibliography{bibliografia}
%

%
\begin{IEEEbiography}[{\includegraphics[width=1in,height=1.25in,clip,keepaspectratio]{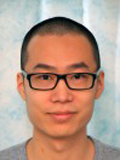}}]{Xuehan Xiong}
received the BS degree in computer science from University of Arizona in 2009, the MS degree in Robotics from Carnegie Mellon University in 2011. He is now working towards the PhD degree in Robotics at Carnegie Mellon University. His research interests include computer vision and optimization. His work has won the best paper award in ISARC 2011 and has been nominated to be the best vision paper in ICRA 2011. His software (http://humansensing.cs.cmu.edu/intraface) on facial feature detection and tracking has been downloaded over 3800 times.
\end{IEEEbiography}

\begin{IEEEbiography}[{\includegraphics[width=1in,height=1.25in,clip,keepaspectratio]{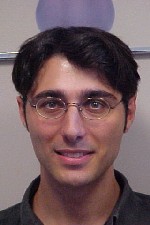}}]{Fernando De la Torre}
received the BSc degree in telecommunications, and the MSc and PhD degrees in electronic engineering from the La Salle School of Engineering at Ramon Llull University, Barcelona, Spain in 1994, 1996, and 2002, respectively. He is an associate research professor in the Robotics Institute at Carnegie Mellon University. 
His research interests are in the fields of computer vision and machine learning. Currently, he is directing the Component Analysis Laboratory (http://ca.cs.cmu.edu) and the Human Sensing Laboratory (http://humansensing.cs.cmu.edu) at Carnegie Mellon University. He has more than 100 publications in refereed journals and conferences. He has organized and co-organized several workshops and has given tutorials at international conferences on the use and extensions of component analysis
\end{IEEEbiography}


\vfill


\setlength{\tabcolsep}{0.5pt}
\begin{figure*}[t]
\centering
\begin{tabular}{cccccccccc}
\includegraphics[width=0.099\linewidth]{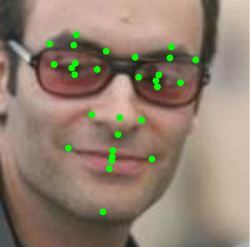}&
\includegraphics[width=0.099\linewidth]{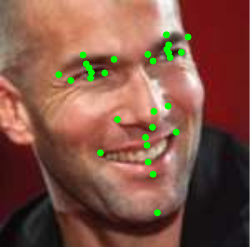}&
\includegraphics[width=0.099\linewidth]{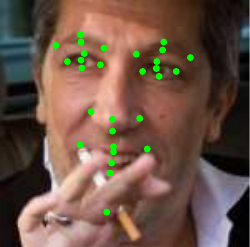}&
\includegraphics[width=0.099\linewidth]{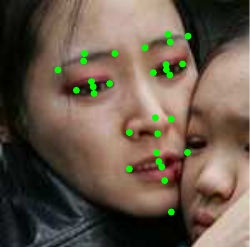}&
\includegraphics[width=0.099\linewidth]{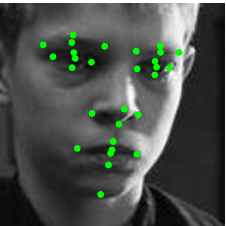}&
\includegraphics[width=0.099\linewidth]{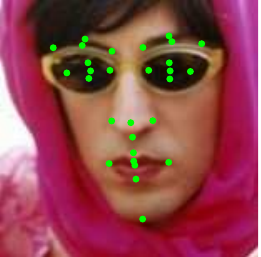}&
\includegraphics[width=0.099\linewidth]{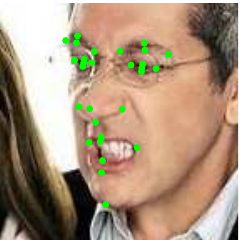}&
\includegraphics[width=0.099\linewidth]{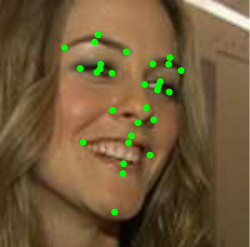}&
\includegraphics[width=0.099\linewidth]{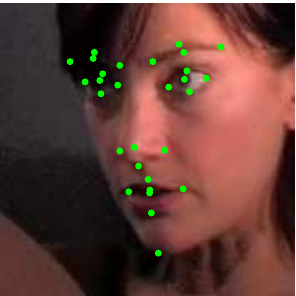}&
\includegraphics[width=0.099\linewidth]{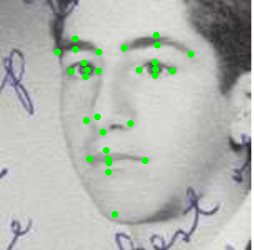}
\\[-2.5pt]
\includegraphics[width=0.099\linewidth]{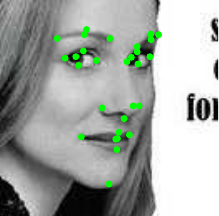}&
\includegraphics[width=0.099\linewidth]{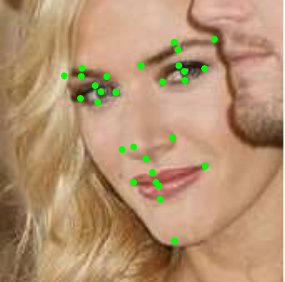}&
\includegraphics[width=0.099\linewidth]{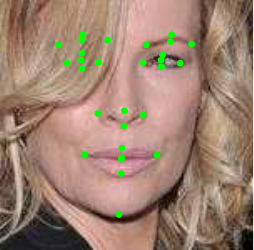}&
\includegraphics[width=0.099\linewidth]{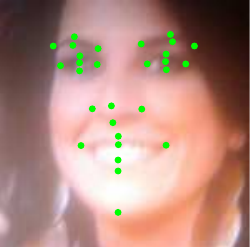}&
\includegraphics[width=0.099\linewidth]{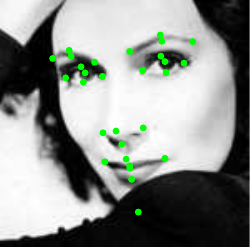}&
\includegraphics[width=0.099\linewidth]{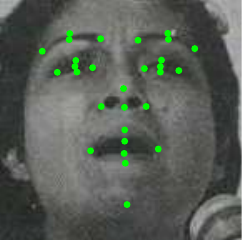}&
\includegraphics[width=0.099\linewidth]{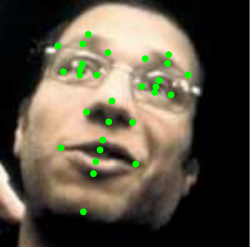}&
\includegraphics[width=0.099\linewidth]{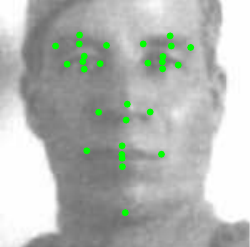}&
\includegraphics[width=0.099\linewidth]{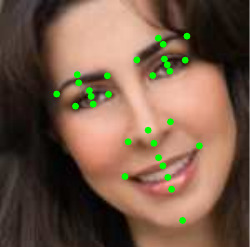}&
\includegraphics[width=0.099\linewidth]{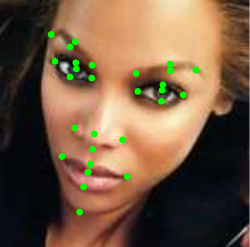}
\\[-2.5pt]
\includegraphics[width=0.099\linewidth]{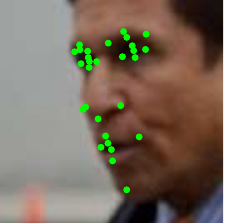}&
\includegraphics[width=0.099\linewidth]{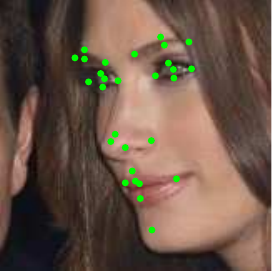}&
\includegraphics[width=0.099\linewidth]{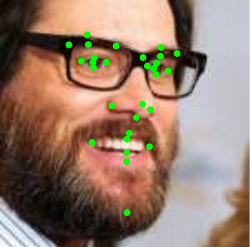}&
\includegraphics[width=0.099\linewidth]{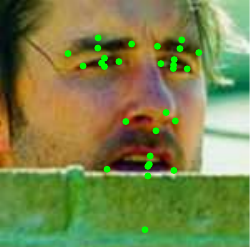}&
\includegraphics[width=0.099\linewidth]{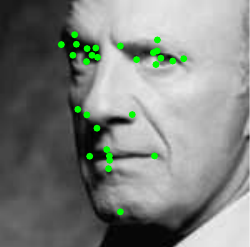}&
\includegraphics[width=0.099\linewidth]{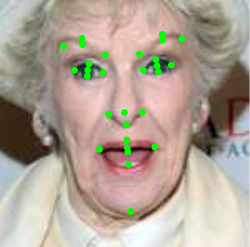}&
\includegraphics[width=0.099\linewidth]{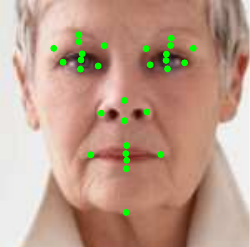}&
\includegraphics[width=0.099\linewidth]{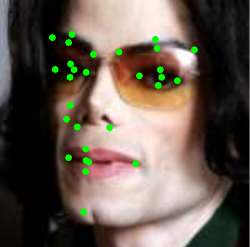}&
\includegraphics[width=0.099\linewidth]{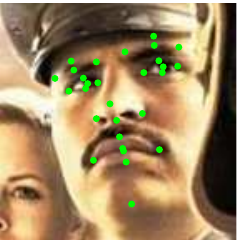}&
\includegraphics[width=0.099\linewidth]{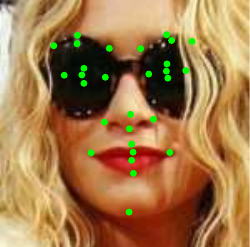}
\end{tabular}
\vspace{-10pt}
\caption{Example results from our method on LFPW dataset.
The first two rows show faces with strong changes in pose and illumination,
and faces that are partially occluded. The last row shows the  10 {\em worst} images w.r.t normalized mean error.}
\vspace{-10pt}
\label{fig::lfpw}
\end{figure*}
\begin{figure*}[t]
\centering
\begin{tabular}{cccccccccc}
\includegraphics[width=0.099\linewidth]{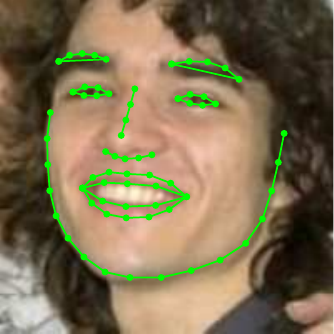}&
\includegraphics[width=0.099\linewidth]{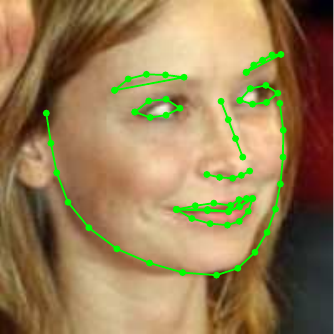}&
\includegraphics[width=0.099\linewidth]{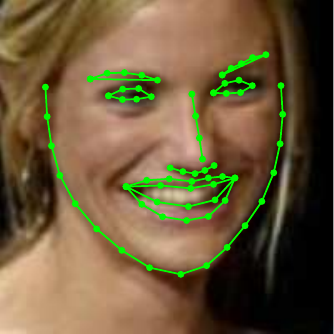}&
\includegraphics[width=0.099\linewidth]{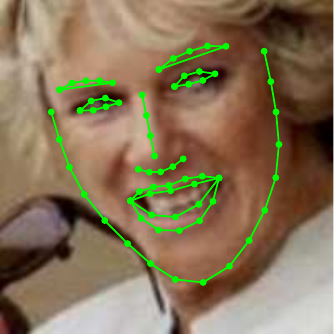}&
\includegraphics[width=0.099\linewidth]{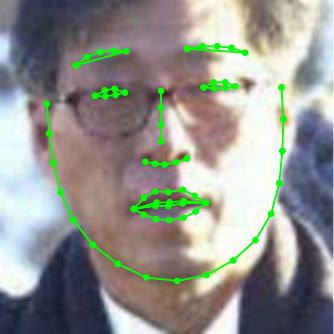}&
\includegraphics[width=0.099\linewidth]{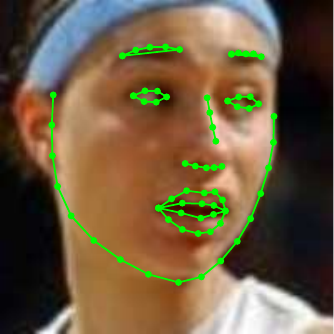}&
\includegraphics[width=0.099\linewidth]{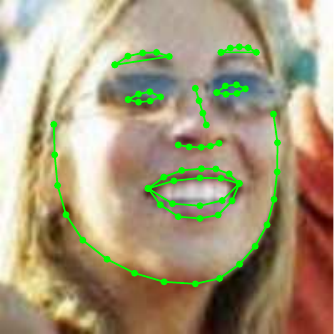}&
\includegraphics[width=0.099\linewidth]{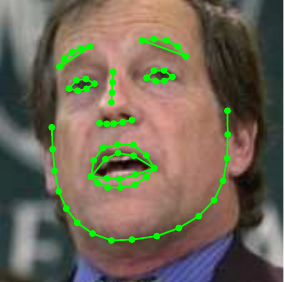}&
\includegraphics[width=0.099\linewidth]{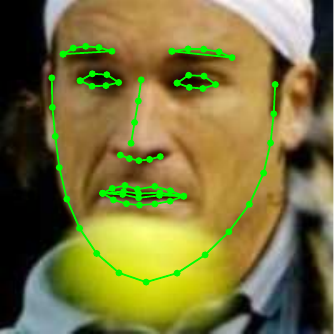}&
\includegraphics[width=0.099\linewidth]{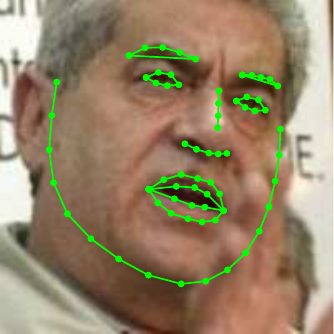}
\\[-2.5pt]
\includegraphics[width=0.099\linewidth]{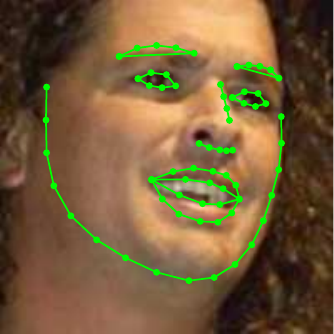}&
\includegraphics[width=0.099\linewidth]{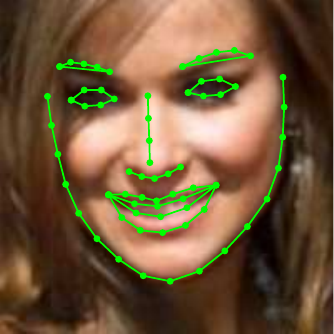}&
\includegraphics[width=0.099\linewidth]{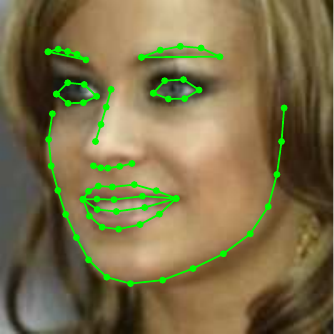}&
\includegraphics[width=0.099\linewidth]{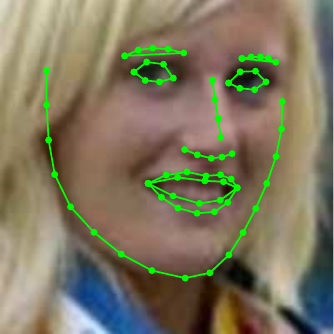}&
\includegraphics[width=0.099\linewidth]{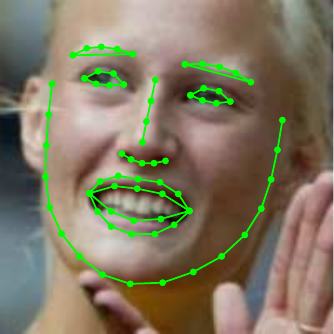}&
\includegraphics[width=0.099\linewidth]{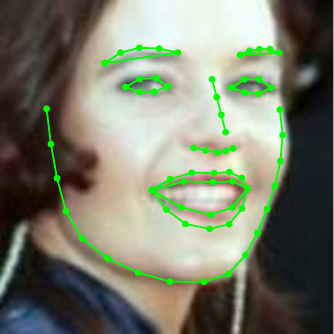}&
\includegraphics[width=0.099\linewidth]{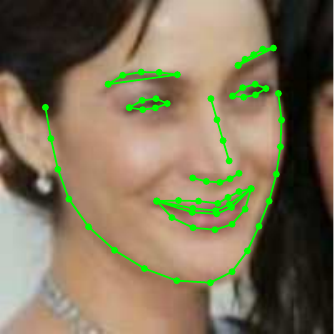}&
\includegraphics[width=0.099\linewidth]{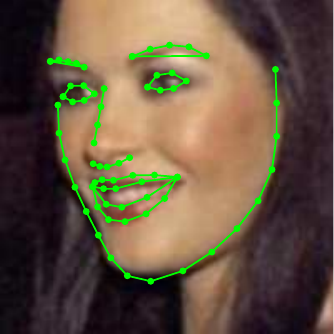}&
\includegraphics[width=0.099\linewidth]{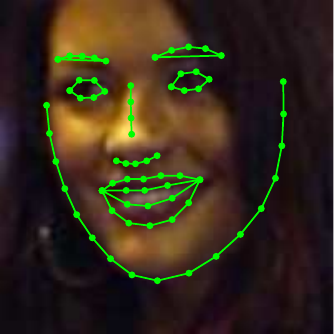}&
\includegraphics[width=0.099\linewidth]{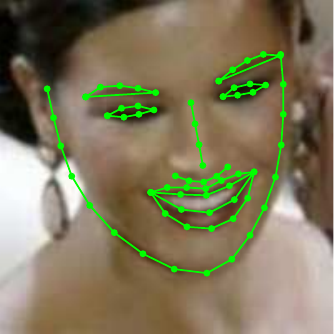}
\\[-2.5pt]
\includegraphics[width=0.099\linewidth]{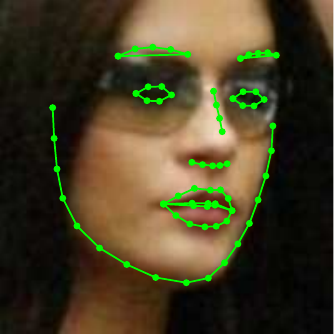}&
\includegraphics[width=0.099\linewidth]{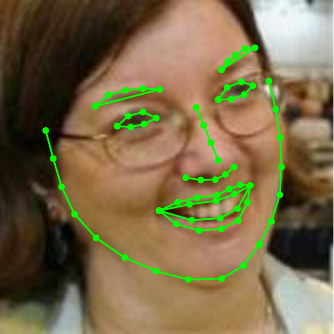}&
\includegraphics[width=0.099\linewidth]{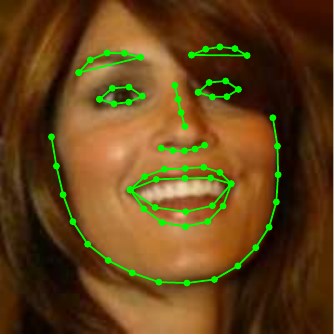}&
\includegraphics[width=0.099\linewidth]{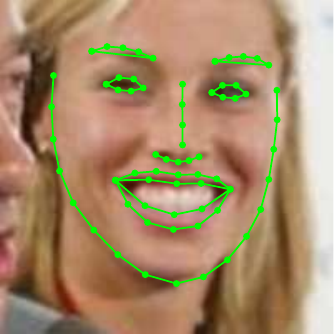}&
\includegraphics[width=0.099\linewidth]{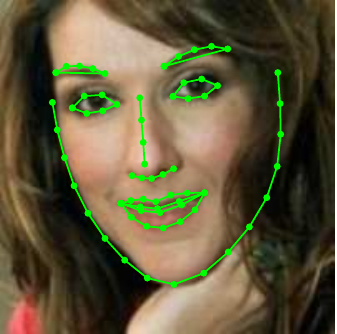}&
\includegraphics[width=0.099\linewidth]{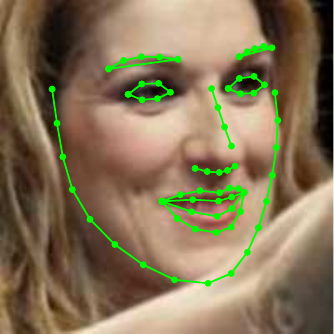}&
\includegraphics[width=0.099\linewidth]{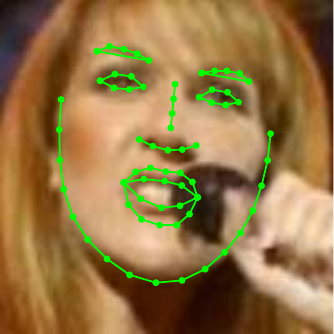}&
\includegraphics[width=0.099\linewidth]{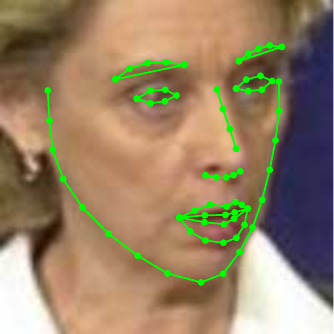}&
\includegraphics[width=0.099\linewidth]{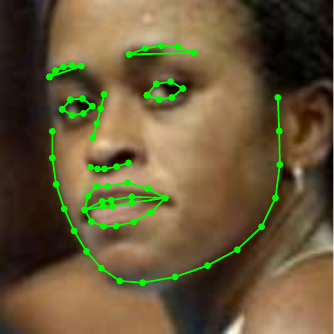}&
\includegraphics[width=0.099\linewidth]{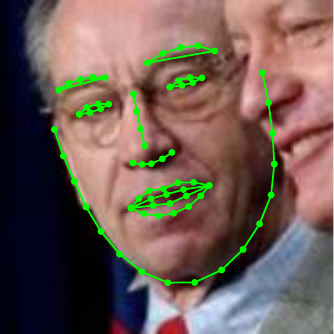}
\\[-2.5pt]
\includegraphics[width=0.099\linewidth]{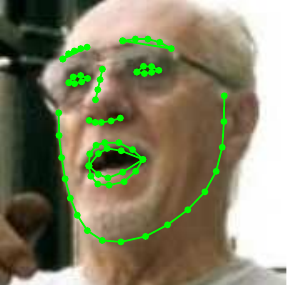}&
\includegraphics[width=0.099\linewidth]{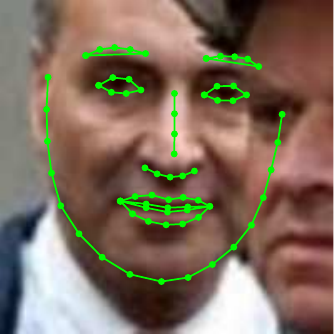}&
\includegraphics[width=0.099\linewidth]{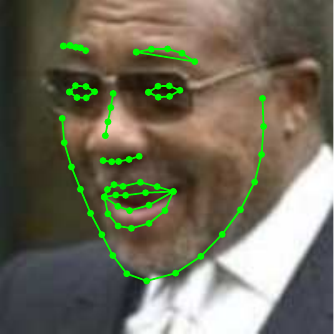}&
\includegraphics[width=0.099\linewidth]{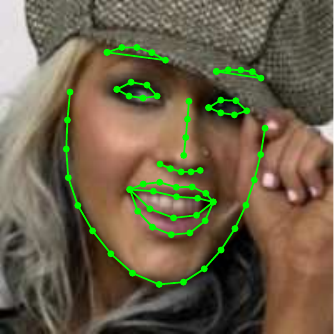}&
\includegraphics[width=0.099\linewidth]{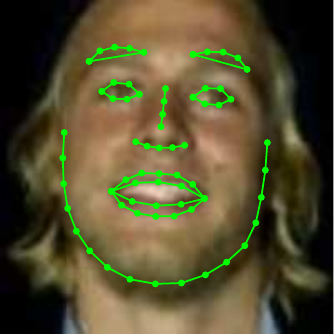}&
\includegraphics[width=0.099\linewidth]{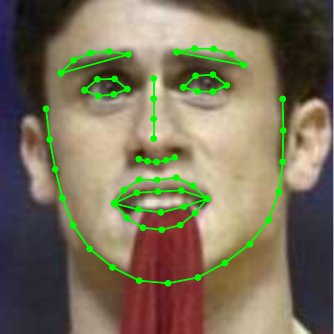}&
\includegraphics[width=0.099\linewidth]{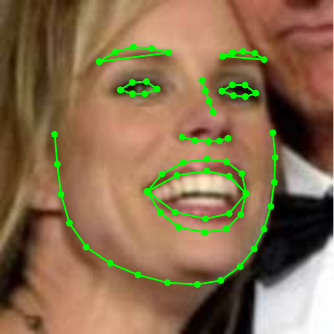}&
\includegraphics[width=0.099\linewidth]{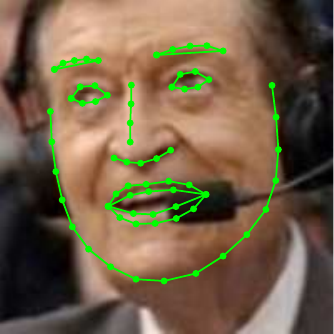}&
\includegraphics[width=0.099\linewidth]{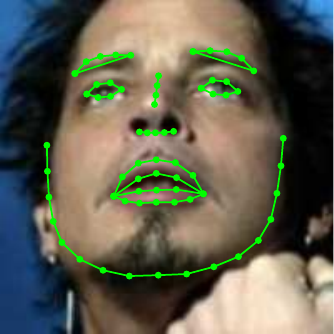}&
\includegraphics[width=0.099\linewidth]{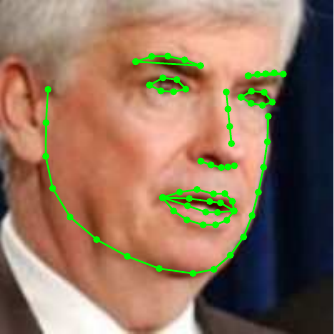}
\end{tabular}
\vspace{-10pt}
\caption{Example results on LFW-A\&C dataset.}
\vspace{-10pt}
\label{fig::lfw}
\end{figure*}
\begin{figure*}[t]
\centering
\begin{tabular}{cccccccccc}
\includegraphics[width=0.099\linewidth]{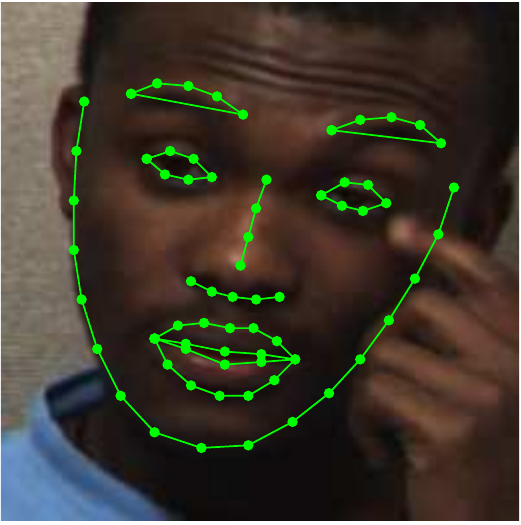}&
\includegraphics[width=0.099\linewidth]{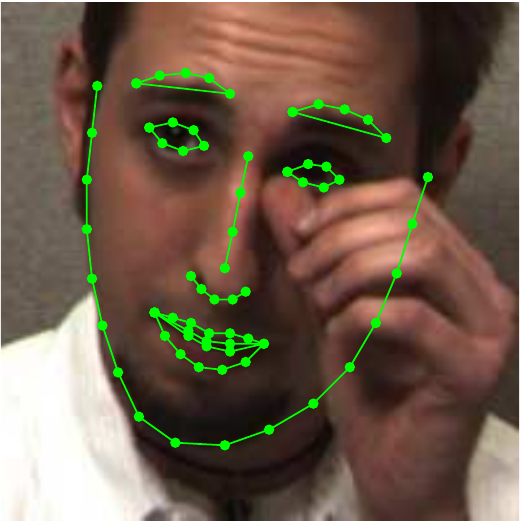}&
\includegraphics[width=0.099\linewidth]{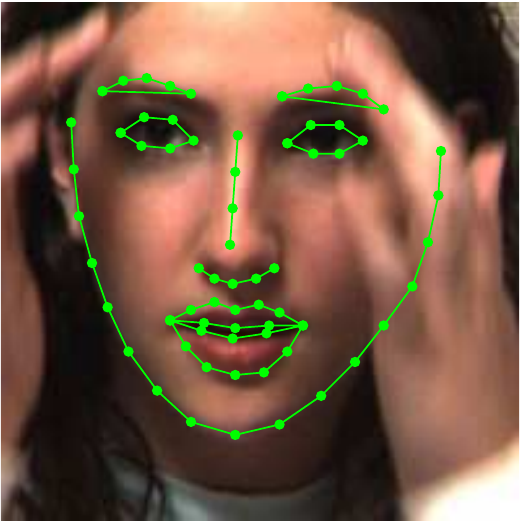}&
\includegraphics[width=0.099\linewidth]{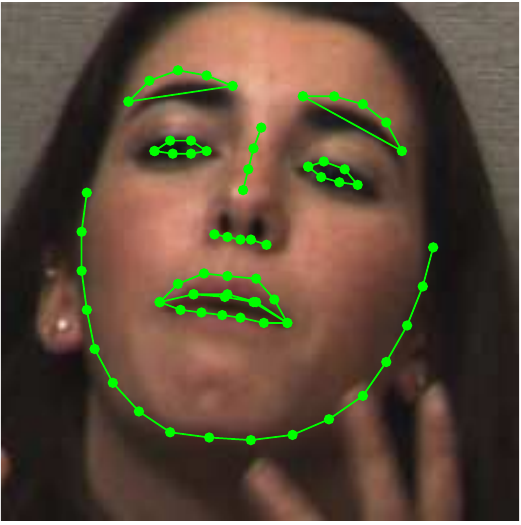}&
\includegraphics[width=0.099\linewidth]{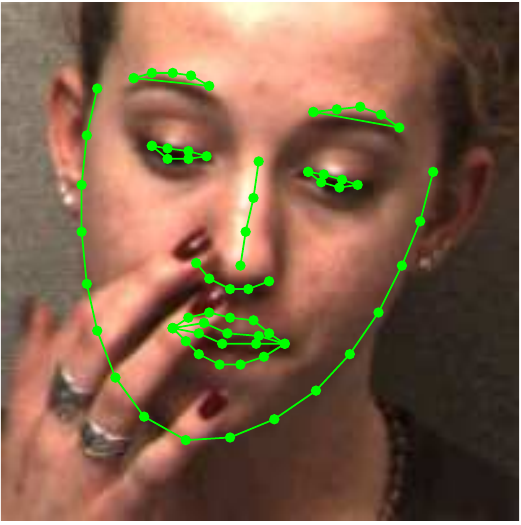}&
\includegraphics[width=0.099\linewidth]{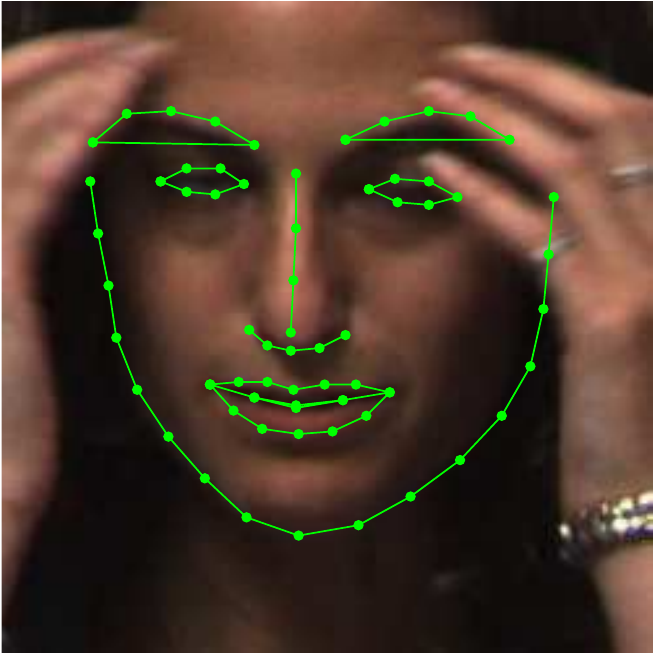}&
\includegraphics[width=0.099\linewidth]{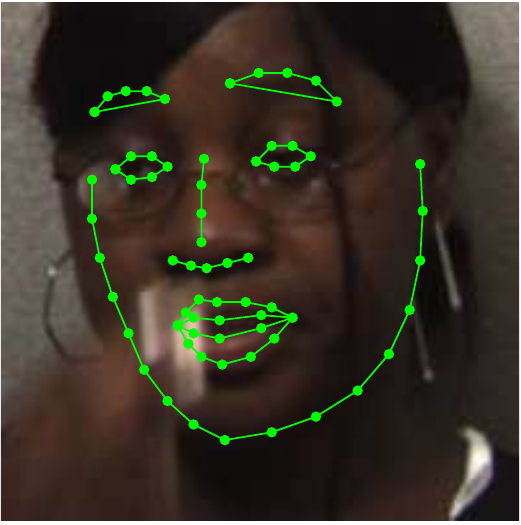}&
\includegraphics[width=0.099\linewidth]{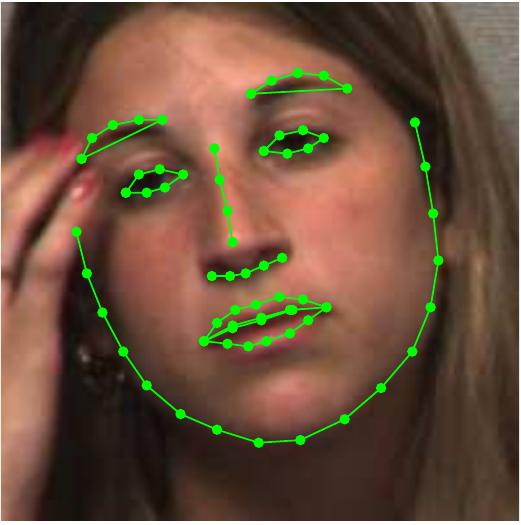}&
\includegraphics[width=0.099\linewidth]{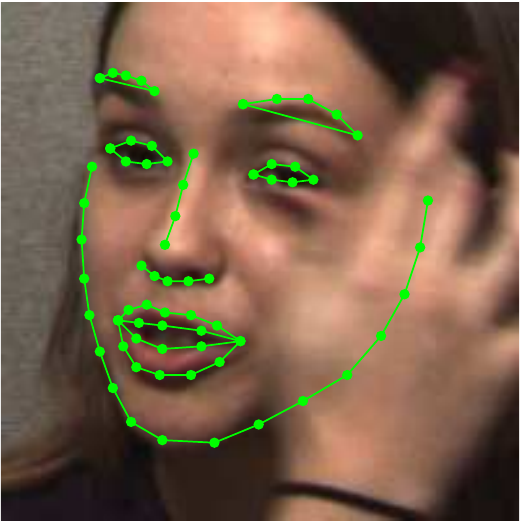}&
\includegraphics[width=0.099\linewidth]{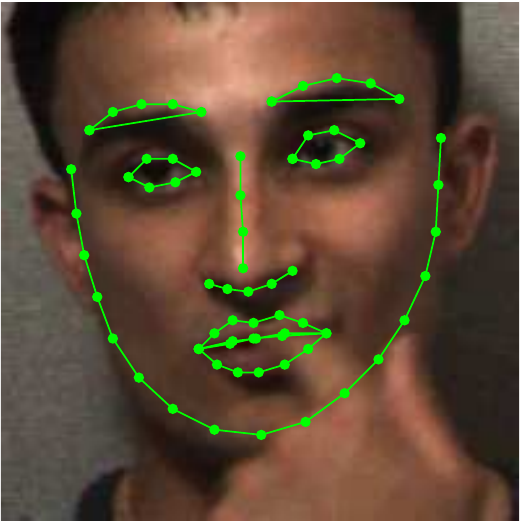}
\\[-2.5pt]
\includegraphics[width=0.099\linewidth]{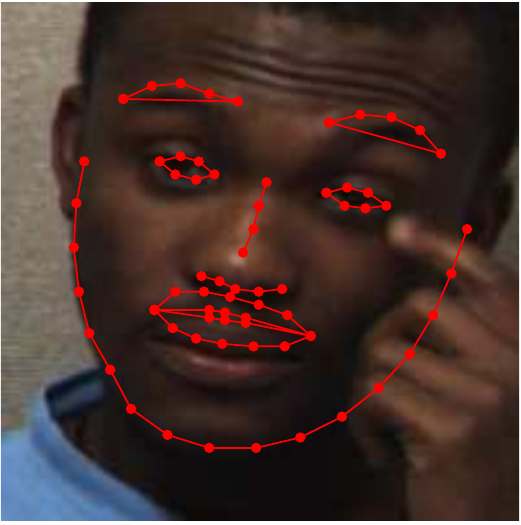}&
\includegraphics[width=0.099\linewidth]{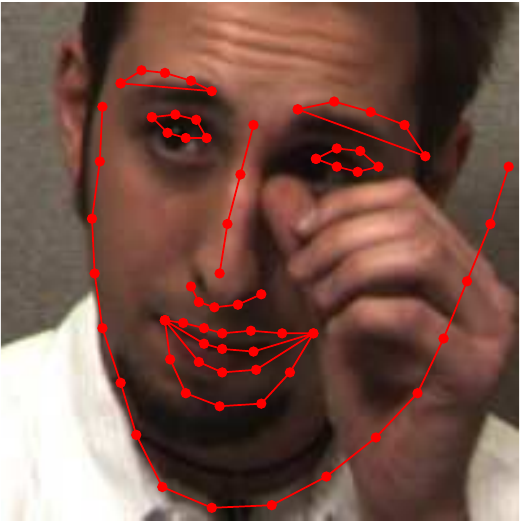}&
\includegraphics[width=0.099\linewidth]{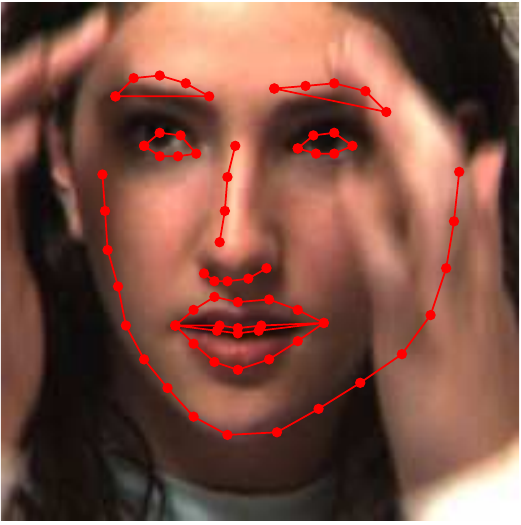}&
\includegraphics[width=0.099\linewidth]{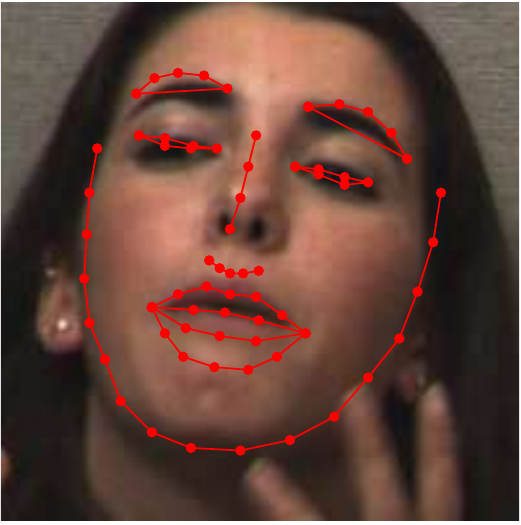}&
\includegraphics[width=0.099\linewidth]{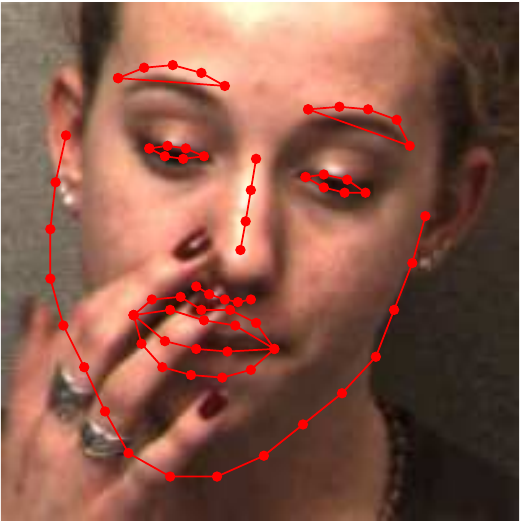}&
\includegraphics[width=0.099\linewidth]{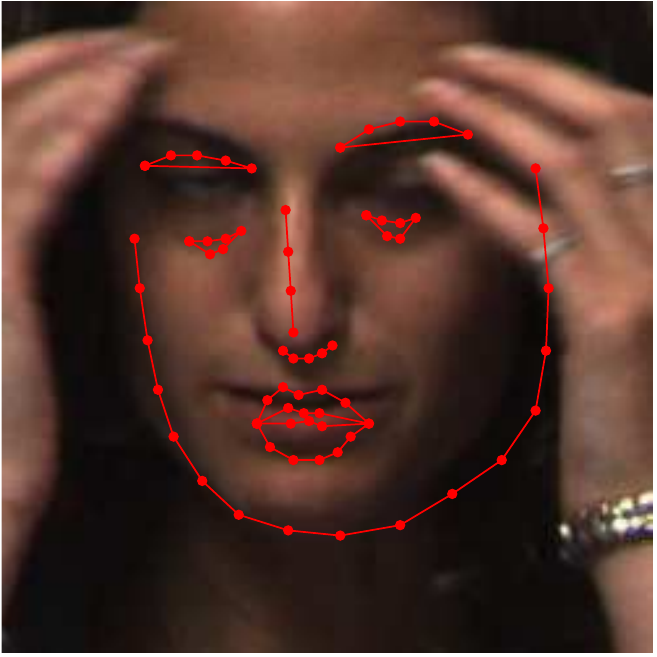}&
\includegraphics[width=0.099\linewidth]{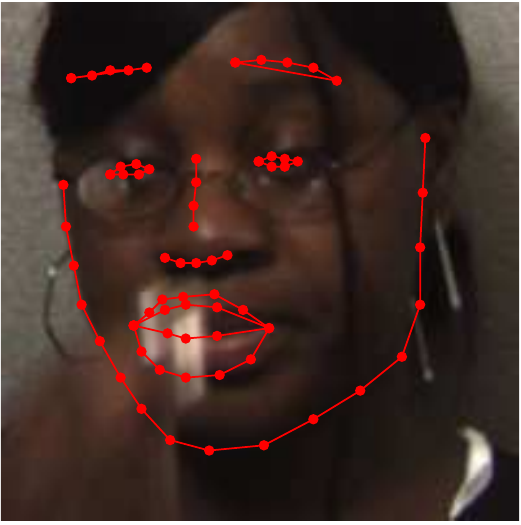}&
\includegraphics[width=0.099\linewidth]{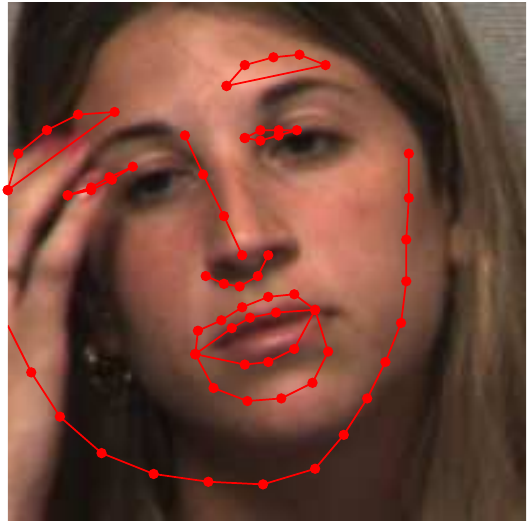}&
\includegraphics[width=0.099\linewidth]{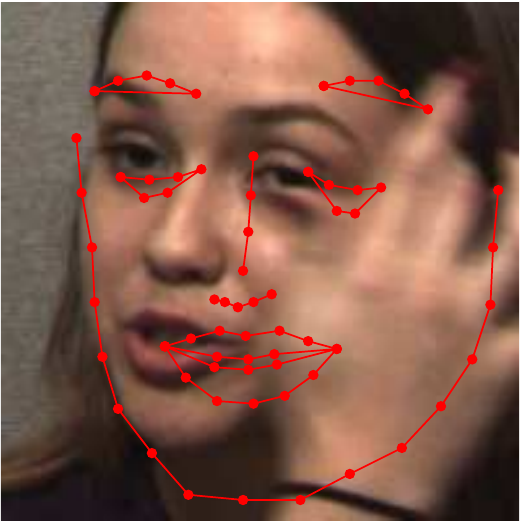}&
\includegraphics[width=0.099\linewidth]{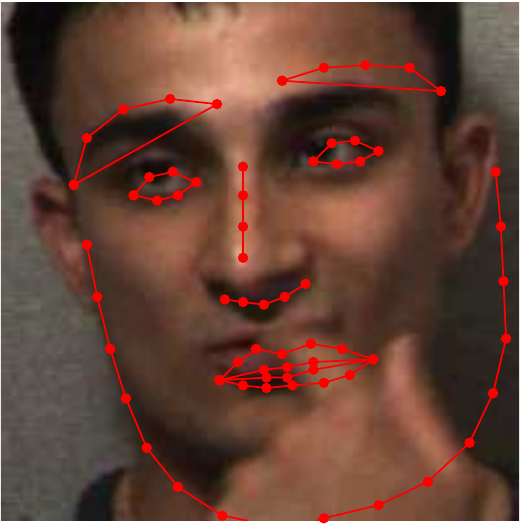}
\end{tabular}
\vspace{-10pt}
\caption{Comparison between the tracking results from SDM (top row)
and person-specific tracker (bottom row).}
\vspace{-10pt}
\label{fig::ru-facs}
\end{figure*}
\begin{figure*}[t]
\centering
\begin{tabular}{cccccccccc}
\includegraphics[width=0.099\linewidth]{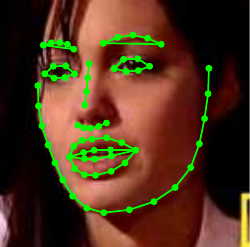}&
\includegraphics[width=0.099\linewidth]{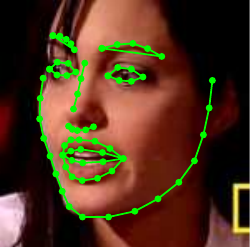}&
\includegraphics[width=0.099\linewidth]{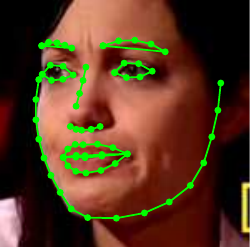}&
\includegraphics[width=0.099\linewidth]{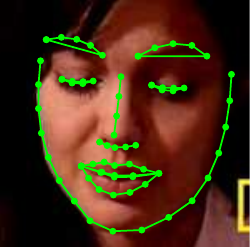}&
\includegraphics[width=0.099\linewidth]{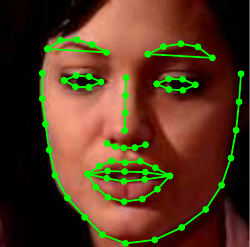}&
\includegraphics[width=0.099\linewidth]{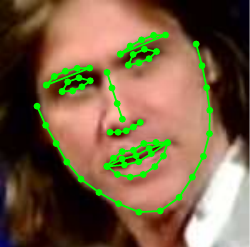}&
\includegraphics[width=0.099\linewidth]{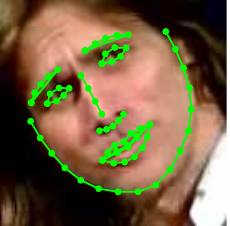}&
\includegraphics[width=0.099\linewidth]{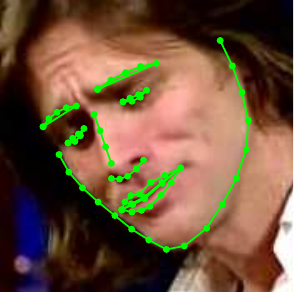}&
\includegraphics[width=0.099\linewidth]{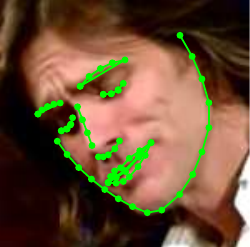}&
\includegraphics[width=0.099\linewidth]{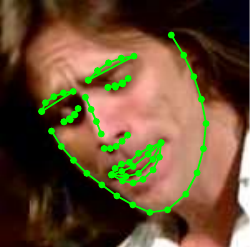}
\\[-2.5pt]
\includegraphics[width=0.099\linewidth]{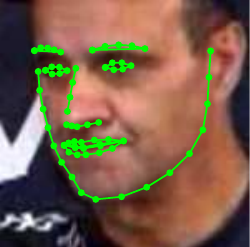}&
\includegraphics[width=0.099\linewidth]{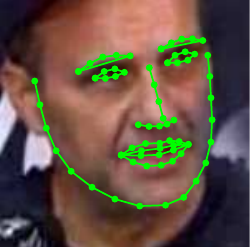}&
\includegraphics[width=0.099\linewidth]{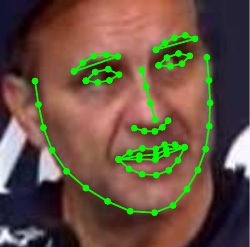}&
\includegraphics[width=0.099\linewidth]{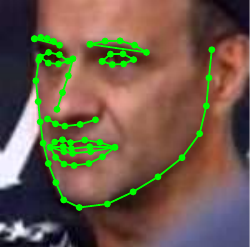}&
\includegraphics[width=0.099\linewidth]{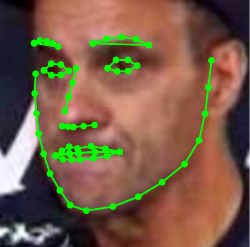}&
\includegraphics[width=0.099\linewidth]{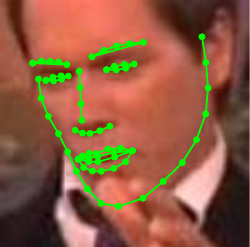}&
\includegraphics[width=0.099\linewidth]{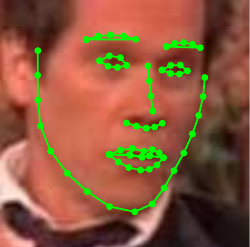}&
\includegraphics[width=0.099\linewidth]{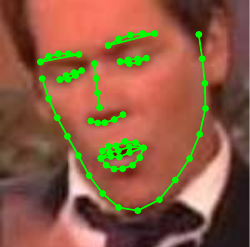}&
\includegraphics[width=0.099\linewidth]{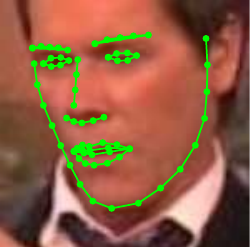}&
\includegraphics[width=0.099\linewidth]{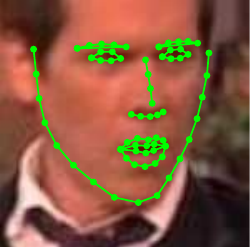}
\end{tabular}
\vspace{-10pt}
\caption{Example results on the Youtube Celebrity dataset.}
\vspace{-10pt}
\label{fig::youtube}
\end{figure*}

\end{document}